\documentclass{article}
\usepackage{dacnotes}
\usepackage{figures/figs}
\usepackage{graphicx, mathrsfs, amssymb, amsfonts, amsmath, xcolor, bbm, natbib, booktabs, tikz,tikz-cd}

\title{Attention Geometry}
\author{Julio Candanedo\footnote{\texttt{julio@sparsetrace.ai}}}
\date{October 2026}

\DeclareFontFamily{OMX}{pxex}{}
\DeclareFontShape{OMX}{pxex}{m}{n}{<-> pxex}{}
\DeclareSymbolFont{pxlargesymbols}{OMX}{pxex}{m}{n}
\DeclareMathSymbol{\intop}{\mathop}{pxlargesymbols}{"52}
\DeclareMathSymbol{\ointop}{\mathop}{pxlargesymbols}{"48}

\begin{document}

\maketitle
\begin{abstract}
Attention normalizes learned scores into a Markov operator. In the continuum, the local sector becomes diffusion on a learned Riemannian manifold, but attention natively bonds distant points, creating nonlocal structure such that no tensor geometry can describe. The Bakry–Émery calculus, built directly from the transition operator, extends geometry to these potentially stratified spaces, giving a coordinate-free account of what vision and language transformers learn.
\end{abstract}
\tableofcontents

\large

\section{Introduction}

Transformers, \citep{vaswani2017attention}, have become the dominant computational architecture across machine learning, from LLMs to generative vision. In the latter, Diffusion Transformers (DiTs) \citep{peebles2022dit, ma2024_sit} integrate attention into diffusion- and flow-based generative modeling \citep{ddpm,debortoli2021dsb,lipman2022_flowMatchingGenerativeModeling, debortoli2022_riemannianScoreBasedGenerativeModelling}. Yet their empirical success has outpaced a unified geometric understanding of their underlying computations. The core component of transformers, \textit{attention}, was developed largely as an architectural mechanism for learning interactions between tokens, while the geometric and dynamical structure encoded by these interactions remains comparatively opaque. This leaves open a more fundamental question: \emph{what geometry, if any, does attention learn, and how does this geometry govern the transport of information?}

The Platonic Representation Hypothesis proposes, based on empirical evidence, that independently trained networks, across architectures, objectives, and even modalities, develop representations that become increasingly aligned as models scale \citep{huh2024_platonicRepresentation}. A complementary premise, the Manifold Hypothesis, holds that high-dimensional data may concentrate near low-dimensional manifolds, \cite{fefferman2016_testingManifoldHypothesis}. These two perspectives motivate asking whether learned representations preserve geometric relations, and whether attention can be understood as an operator governing the transport of information over them. A geometric theory of attention could therefore reveal the structure underlying these learned interactions and offer a route to interpretability beyond purely empirical descriptions of the layer's behavior.

Diffusion maps (DMAP\footnote{Let's call in Diffusion Manifold Approximation Projection à la \cite{mcinnes2018_umapSoftware}.}) \citep{scholkopf1998_kpca, BelkinNiyogi2003_LaplacianEigenmaps,coifman2006_diffusionMaps,singer2008_mahalanobisDiffusionMaps, abanov2025_quantumGeometryData, gottwald2025_stableSB}
offer a particularly well-behaved foundation for manifold learning.
Their positive-semidefinite (PSD) Gauss kernels normalize into reversible Markov operators with stationary measures, whose self-adjoint Hilbert-space realizations admit Fourier harmonics, extended off-sample by the Nystr\"om extension. Under suitable continuum limits, these operators converge to a Witten Laplacian, whose Bakry--\'Emery carr\'e du champ recovers the Riemannian geometry, the Dirichlet form, and a partridge in a pear tree.
This makes DMAP a natural starting point for a reductionist theory of attention, the key component of the transformer.

The transformer's central innovation is \emph{contextual} computation, in which a sequence of $N$ tokens forms a representation $\mathcal R\in\mathbb R^{N\times D}$ with embedding dimension $D$. Its architecture composes two principal operations across depth: Multi-Head Self-Attention (MHSA), which aggregates information across tokens through multiple attention heads, and a Feed-Forward Neural Network (FFN), which applies a nonlinear feature map independently to each token,
\begin{align}\label{vit}
\scalebox{0.85}{\ViTBlock}
\end{align}
with residual and normalization layers omitted for clarity. For a single attention head, trained projections define the queries $\mathcal Q=\mathcal RW_Q$, keys $\mathcal K=\mathcal RW_K$, and values $\mathcal V=\mathcal RW_V$. Attention is then the row-stochastic operator, $\mathsf A\mathbf1=\mathbf1$,
\begin{align}\label{eq:attention-markov}
\mathsf A=\mathrm{softmax}(\beta\mathcal Q\mathcal K^\top)=\mathrm{rownorm}\!\left(\exp_\odot(\beta\mathcal Q\mathcal K^\top)\right),
\end{align}
which updates the values through $\mathcal V\mapsto\mathsf A\mathcal V$. This normalized-kernel operation admits both a statistical interpretation as Nadaraya--Watson regression \citep{nadaraya1964_kernelRegression,watson1964_kernelRegression,tsai2019_transformerDissection, teo2024kernelpcaattention} and a dynamical interpretation as a finite Markov transition operator.

Our geometric perspective regroups the Transformer operations, associating each attention layer with the FFN that immediately precedes it across the conventional block boundary,
\begin{align}\label{tiv}
\scalebox{0.85}{\TiVBlock}
\end{align}
Under this convention, the FFN acts as a nonlinear feature map on context-dependent representations, while attention transports information according to learned pairwise relations in that feature space. This pairing provides the starting point for investigating the geometric structure underlying attention, using diffusion maps as a mathematically controlled reference.

We approach general attention through a hierarchy of geometric ablations, progressively relaxing the restrictions of DMAP. DMTN\footnote{Let's call this tentative theory, Diffusion-Manifold  Transport and Non-reciprocity.} introduces directed transport, ATAP\footnote{Lets call this tentative theory, Attention Topology Approximation Projection.} permits indefinite symmetric bonds, and full attention (ATTN) combines both. Learned affinities may connect distant regions of the underlying manifold, giving rise, under suitable continuum scaling, to nonlocal jump processes. The Beurling--Deny decomposition characterizes such interactions in the reversible setting, while Courr\`ege's theory unifies local diffusion, directed drift, and nonlocal jumps within a single generator. Consequently, conventional tensorial differential geometry can describe the local sector but cannot, by itself, characterize the full geometry of attention. The Bakry--\'Emery calculus \citep{bakry2014_markovDiffusionOperators}, formulated directly through Markov generators and their \textit{carr\'e-du-champ}, provides a more general framework for investigating this learned geometry.

The continuum perspective suggests two distinct limits: within a single attention layer, the large-sample and localization limit yields a geometric generator; across network depth, successive blocks define a sequence of learned geometries,
\begin{align*}
\mathscr M^{(0)}\xrightarrow{F_0}\mathscr M^{(1)}\xrightarrow{F_1}\cdots\xrightarrow{F_{L-1}}\mathscr M^{(L)}.
\end{align*}
Since Transformer blocks employ residual updates, $F_\ell=\mathrm{id}+f_\ell$, their composition suggests a continuous-depth description when the updates are appropriately scaled. In such a limit, one may seek an evolution equation for suitable observables,
\begin{align*}
\partial_s u=-\mathscr C_s u,
\end{align*}
where $\mathscr C_s$ is a depth-dependent positive geometric generator. Related perspectives have appeared in graph neural diffusion \citep{chamberlain2021_graphNeuralDiffusion} and interacting-particle descriptions of attention \citep{geshkovski2023_selfAttentionClusters, alcalde2025_attentionFrankWolfe}. Establishing this second continuum limit could connect the geometry of attention to PDEs, geometric flows, and potentially renormalization-group descriptions of deep networks.
\section{Discrete Diffusion}

Suppose we have a collection of finite samples $\{R\}=\mathcal{R}$, and we would like to know the structure of this dataset. That is how samples $R,R'\in\mathcal{R}$ are related to each other, and how to generate new samples. We may instinctively reach for a distance measure in the ambient space $||R-R'||$, however it is well known empirically that this is not a good surrogate for intrinsic similarity $\mathscr{S}(R,R')$ directly. This can be clearly illustrated in a simple example, involving the well-known Swiss-Roll dataset.

Two points can be close in distance, and simultaneously be far apart intrinsically, Fig.~\ref{fig:einstein-swiss-roll}. Hence our question truly becomes how to bootstrap their intrinsic relation. In order to make progress we reach for a locality principle, built on an intuitive notion that {for any data samples there should be some that are very close}. {Essentially that we can fill in our dataset to arrive to a continuous-like structure of a Riemannian manifold} $\mathscr{R}$.
Importantly, this does not give us a manifold, we merely presuppose its existence.

{For now, in the discrete choose} the pairwise Gaussian bond cost $\mathcal H=\lVert\mathcal{R}-\mathcal{R}^\top\rVert^2$ and collect its positive symmetric conductances in a matrix,
\begin{align}\label{eq:einstein-gaussian-conductance}
\mathsf{K}&=\exp_\odot(-\beta\mathcal H).
\end{align}
The row sums $\mathsf{p}=\mathsf{K}\mathbf1$ provide an unnormalized kernel-density estimate (KDE) on the finite samples. They define Kirchhoff's graph Laplacian and its Dirichlet energy,
\begin{align}\label{eq:einstein-graph-energy}
{\mathsf{L}}_K &= \mathrm{diag}(\mathsf{p})-\mathsf{K}  \\
\mathcal E(f,f) &=f^\top {\mathsf{L}}_K f =\frac12\sum_{i,j}\mathsf{K}_{ij}(f_i-f_j)^2.
\end{align}
The energy penalizes differences between strongly coupled samples. Its definition uses neither geodesics nor {a metric tensor}, and everything is evaluated on the finite graph.
To make the conductance into a Markov walk, normalize its rows,
\begin{align}\label{eq:einstein-markov-laplacian}
\mathsf{P}&=\mathrm{diag}(\mathsf{p})^{-1}\mathsf{K},\\
\mathsf{L}&=\mathsf I-\mathsf P.
\end{align}
A discrete Markov operator is a nonnegative row-stochastic matrix satisfying $\mathsf P\mathbf1=\mathbf1$. Its entries are one-step transition probabilities, and it acts by weighted averaging on functions of the samples. The finite Laplacian measures the discrepancy between a function and its one-step average, $(\mathsf Lf)_i=\sum_j\mathsf P_{ij}(f_i-f_j)$. The forward evolution of a probability mass $\mu$ is $\mu^{(t+1)\top}=\mu^{(t)\top}\mathsf P$, while $\mathsf P$ acts on column observables. Symmetry of $\mathsf{K}$ gives the exact stationary law $\pi=\mathsf{p}/(\mathbf1^\top\mathsf{p})$ and detailed balance,
\begin{align}\label{eq:einstein-detailed-balance}
\mathrm{diag}(\pi)\,\mathsf{P}&=\mathsf P^\top\mathrm{diag}(\pi).
\end{align}

\begin{figure}[h!]
  \centering
  \begingroup\setlength{\fboxsep}{0pt}{\includegraphics[width=\linewidth]{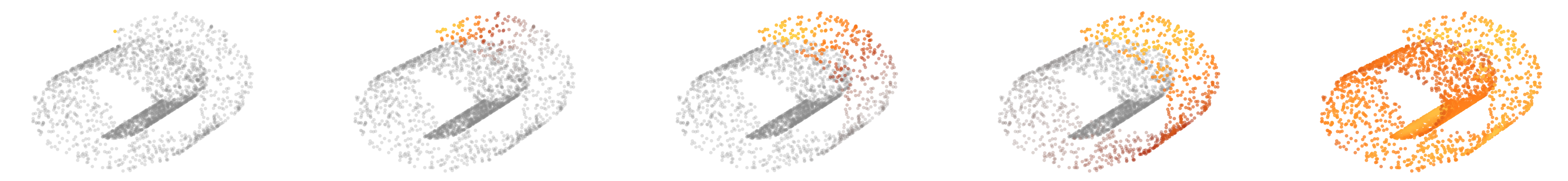}}\endgroup
  \caption{Heat spreading on a sampled Swiss roll under a Markov walk. At suitable bandwidth the motion follows local neighborhoods of the sheet, and the long-time equilibrium is visibly nonuniform.}
  \label{fig:einstein-swiss-roll}
\end{figure}

Figure~\ref{fig:einstein-swiss-roll} shows the walk in action. Early heat remains on the local sheet rather than short-circuiting through neighboring turns. The later panels show relaxation across the sampled geometry. The final panel anticipates the continuum calculation by displaying a nonuniform equilibrium.

{For our Laplacian-operator,} $\mathsf{L}$, {its modes $\phi$ achieved from diagonalization are harmonics.} {These are not what we originally desired, geodesics and coordinates, but lets use the method of Fourier in solving the heat equation.}
The relevant harmonics diagonalize $\mathsf L$ in the stationary inner product $\langle f|g\rangle_\pi=f^\top\mathrm{diag}(\pi)g$. Choose $\langle\phi_j|\phi_k\rangle_\pi=\delta_{jk}$ and expand $f=\sum_k c_k\phi_k$. The finite Dirichlet energy is then the spectral energy of the Laplacian,
\begin{align}\label{eq:einstein-spectral-energy}
\mathsf L\phi_k&=\lambda_k\phi_k,\\
\mathcal{E}(f,f)&=(\mathbf1^\top\mathsf{p})\langle f|\mathsf L|f\rangle_\pi=(\mathbf1^\top\mathsf{p})\sum_k\lambda_k|c_k|^2.
\end{align}
In particular, $\langle\phi_k|\mathsf L|\phi_k\rangle_\pi=\lambda_k$, hence each normalized diffusion harmonic has energy equal to its Laplacian eigenvalue in the stationary inner product. The factor $\mathbf1^\top\mathsf{p}$ restores the unnormalized Kirchhoff energy.
The heat flow of the relative density $u=\mu/\pi$, with division taken componentwise, has the spectral form:
\begin{align*}
u^{(t)}=\sum_k{c_k}(1-\lambda_k)^t\phi_k.
\end{align*}
The low-frequency modes resolve slow, large-scale transport, much as Fourier harmonics resolve the ordinary heat equation. These modes give us a useful representation of the sampled space.
The big question is whether this discrete representation is faithful to {true diffusion on a real Riemannian-manifold.}
This is Einstein's continuum question, \cite{einstein1905_brownianMotion}. Take $N\&\beta\to\infty$ together, {with sufficiently many samples inside the shrinking Gaussian neighborhood} of width $\beta^{-1/2}$. The finite Markov operator approaches an integral transition operator $\mathscr P_\beta$, and its {rescaled positive generator} is:
\begin{align}\label{eq:einstein-continuum-generator}
\mathscr{C}_\beta f &=4\beta\left(f-\mathscr P_\beta f\right),\\
(\mathscr P_\beta f)(x)  &=\int_{\mathscr R}d\nu(x')\,\mathscr P_\beta(x,x')f(x').
\end{align}
Here $d\nu(x')=\mathsf{p}(x')\,dV_{\mathscr G}(x')$, with $\mathsf{p}(x)$ the continuum sampling density relative to the Riemannian volume of ${\mathscr G}$. {The finite KDE satisfy} ${\mathsf{p}_i\sim N(\pi/\beta)^{d/2}\mathsf{p}(R_i)}$ on a $d$-dimensional manifold, away from the boundary, so the continuum density is obtained after removing the common bandwidth $\beta$ and sample-count $N$ factors.
Einstein's small-displacement expansion produces the Laplace--Beltrami operator together with a sampling-induced drift. For the uncorrected Gaussian normalization, away from boundary effects and under the usual smoothness and sampling assumptions,
\begin{align}\label{eq:einstein-weighted-continuum-limit}
\mathscr{C}_\beta f &\longrightarrow -\Delta_{\mathscr G} f-2\langle d\log\mathsf{p}|df\rangle_{\mathscr G}.
\end{align}
The finite walk had $\pi\propto\mathsf{p}$, and its continuum equilibrium density relative to Riemannian volume is proportional to $\mathsf{p}^2$, thus this operator remembers how the manifold was sampled.
\section{Cauchy--Green: Learned Geometry}\label{sec:cauchygreen}\label{sec:cauchy-green}

The preceding construction relied on the Euclidean geometry of the sampled space. We now generalize this through a learned PSD quadratic form. Let $\mathscr R$ be a smooth $d$-dimensional manifold embedded in $\mathbb R^D$ by $\iota:\mathscr R\to\mathbb R^D$. The sampled points $r_i\in\mathscr R$ correspond to ambient observations $R_i=\iota(r_i)\in\mathcal R\subset\mathbb R^D$, collected as rows $R_i^\top$ of $\mathcal R\in\mathbb R^{N\times D}$. Introducing the feature Mahalanobis-metric $G=W_+W_+^\top$, wrapped in context yields the Gram-metric $\mathcal G=\mathcal R G\mathcal R^\top$, and through the law of cosines, with $n = \mathrm{diag}(\mathcal G)$, we obtain a distance measure:
\begin{align}\label{eq:wng-learned-bond}
\mathcal H &= n\mathbf1^\top+\mathbf1 n^\top-2\mathcal G.
\end{align}
In the continuum, this form induces a metric on $\mathscr R$ through the generalized right Cauchy--Green pullback. In local coordinates $x=(x^1,\ldots,x^d)$, with embedding Jacobian $J_\iota(x)=\partial\iota(x)/\partial x$, we obtain:
\begin{align}\label{eq:wng-pullback-metric}
\mathscr G_{ab}(x)&=\partial_a\iota(x)^\top G\,\partial_b\iota(x),\\
\mathscr G(x)&=J_\iota(x)^\top GJ_\iota(x).
\end{align}
Even for constant $G$, the induced metric $\mathscr G(x)$ generally varies with position. At fixed embedding, learning $G$ produces the deformation:
\begin{align}
\delta\mathscr G(x)
=J_\iota(x)^\top\delta G\,J_\iota(x),
\end{align}
thereby modifying intrinsic lengths, volumes, and potentially curvature.
Feature transformations provide a complementary mechanism. Consider the diagonal affine map $F(R)=\mathrm{diag}(\gamma)R+b$, with $\gamma,b\in\mathbb R^D$. For parameters shared across observations, the induced quadratic form becomes
\begin{align}
G_\gamma&=\mathrm{diag}(\gamma)\,G\,\mathrm{diag}(\gamma),
\end{align}
while the translation $b$ cancels from pairwise squared distances. More generally, a nonlinear pointwise feature map $F:\mathbb R^D\to\mathbb R^D$, such as a feed-forward network (FFN), deforms the embedding to $F\circ\iota$. With Jacobian $J_F(R)=\partial F(R)/\partial R$, the induced metric is:
\begin{align}\label{eq:cg-general}
\mathscr G_F(x)
&=J_\iota(x)^\top J_F(\iota(x))^\top G\,
J_F(\iota(x))J_\iota(x).
\end{align}
Thus, learning either $G$ or the feature representation $F$ deforms the intrinsic geometry. This is an active deformation when the ambient quadratic form is held fixed.

A Riemannian metric requires the pullback to be positive definite on tangent spaces, while a purely local diffusion limit additionally requires the learned bond to separate distant points. Under these conditions, the short-displacement bond is governed by $\mathscr G_{ab}(x)\,dx^a dx^b$. For the associated continuum transition kernel $\mathscr P_\beta$, the rescaled second moment gives the local diffusion tensor,
\begin{align}\label{eq:wng-diffusion-tensor}
D^{ab}(x)
&=\lim_{\beta\to\infty}2\hbar\beta
\int_{\mathscr R}d\nu(y)\,\mathscr P_\beta(x,y)
(y^a-x^a)(y^b-x^b)\\
&=\hbar\,\mathscr G^{ab}(x).
\end{align}
Here $x$ and $y$ are compared in a common local chart, with sufficiently dense sampling as the bandwidth shrinks. The instantaneous quadratic variation per unit diffusion time is $2\hbar\mathscr G^{-1}$.

The deformation also changes the intrinsic volume form and the density relative to it. For a fixed coordinate sampling density $q(x)\,d^dx$,
\begin{align}\label{eq:wng-density-volume}
d\nu(x)&=\rho_{\mathscr G}(x)\,dV_{\mathscr G}(x),\\
dV_{\mathscr G}(x)&=\sqrt{\det\mathscr G(x)}\,d^dx,\\
\rho_{\mathscr G}(x)&=\frac{q(x)}{\sqrt{\det\mathscr G(x)}}.
\end{align}
Away from boundaries, the finite degree satisfies $\mathsf p_i\sim N(\pi/\beta)^{d/2}\rho_{\mathscr G}(x_i)$ in the dense-sampling regime. Hence, a learned feature-space deformation changes both the local diffusion covariance and the sampling density relative to the induced geometry.
\section{Witten--Novikov--Girsanov: Learned Transport}\label{sec:einstein-witten-girsanov}

The Cauchy--Green construction determines the learned metric $\mathscr G$ and its associated diffusion tensor. We now hold this geometry fixed and examine the transport it supports. Under the continuum scaling of Einstein's walk, with diffusion-time constant $\hbar>0$, the reversible Markov operator converges, \citep{singer2006_graphToManifoldLaplacian, hein2007_graphLaplaciansRandomNeighborhoodGraphs}, to:
\begin{align}\label{eq:wng-weighted-generator}
4\hbar\beta(I-\mathscr P_\beta)f &\longrightarrow
-\hbar\Delta_{\mathscr G}f
-2\hbar\langle d\log\rho|df\rangle_{\mathscr G}.
\end{align}
Here $\rho$ denotes the sampling density relative to the learned Riemannian volume. Introducing the free energy $\mathscr F=-\hbar\log\rho$, we obtain the positive generator:
\begin{align}\label{eq:wng-witten-markov}
\mathscr C_{\mathrm W}f
&=-\hbar\Delta_{\mathscr G}f
+2\langle d\mathscr F|df\rangle_{\mathscr G}.
\end{align}
The exact one-form $d\mathscr F$ generates conservative transport on the learned geometry. Witten's construction, \citep{witten1982_supersymmetryMorse, hori2003_mirrorSymmetry}, incorporates this potential through a deformation of the exterior derivative,
\begin{align}\label{eq:wng-witten-differential}
d_{\mathscr F,\hbar}
&=\hbar e^{-\mathscr F/\hbar}d\,e^{\mathscr F/\hbar}
=\hbar d+d\mathscr F\wedge,\\
\mathscr{H}_{\mathscr F,\hbar}
&=d_{\mathscr F,\hbar}^\dagger d_{\mathscr F,\hbar}\\
&=-\hbar^2\Delta_{\mathscr G}
+|d\mathscr F|_{\mathscr G}^2
-\hbar\Delta_{\mathscr G}\mathscr F.
\end{align}
The resulting symmetric Witten Laplacian admits the formal positive zero mode $\psi_0=e^{-\mathscr F/\hbar}$. Its ground-state Doob transform recovers the Markov generator and its invariant density,
\begin{align}\label{eq:wng-witten-doob}
\mathscr C_{\mathrm W}f
&=\hbar^{-1}\psi_0^{-1}
\mathscr{H}_{\mathscr F,\hbar}(\psi_0f),\\
\varpi(x)
&=Z^{-1}\psi_0(x)^2
=Z^{-1}e^{-2\mathscr F(x)/\hbar}.
\end{align}
For the sampling potential, $\psi_0=\rho$, giving $\varpi\propto\rho^2$. The diffusion is reversible in $L^2(\varpi\,dV_{\mathscr G})$, with Dirichlet energy $\hbar\int_{\mathscr R}\langle df|df\rangle_{\mathscr G}\varpi\,dV_{\mathscr G}$.

The finite construction, \citep{bianconi2021_higherOrderNetworks}, follows from the node-to-edge exterior derivative $(\mathsf df)_{i\to j}=f_j-f_i$. Given a positive node field $h_i=e^{-\mathcal F_i/\hbar}$ and a symmetric kernel $\mathsf K$, scalar reweighting produces:
\begin{align}\label{eq:wng-finite-witten}
\mathsf P_h
&=\mathrm{diag}(\mathsf Kh)^{-1}
\mathsf K\,\mathrm{diag}(h),\\
\pi^{(h)}
&\propto h\odot(\mathsf Kh).
\end{align}
The corresponding reversible conductances give the finite Dirichlet energy:
\begin{align}
\mathcal E_h(f,f)&=\frac12\sum_{i,j}
\pi_i^{(h)}(\mathsf P_h)_{ij}|(\mathsf df)_{i\to j}|^2.
\end{align}
Detailed balance follows from the symmetry of $\mathsf K$. In the continuum, a smooth scalar reweighting contributes an additional potential to $\mathscr F$. The Coifman--Lafon density correction provides a particular choice that removes the sampling-induced drift.

Directed transport is obtained by extending the exact differential to a general real one-form $\mathscr A$,
\begin{align}\label{eq:wng-general-transport}
\mathscr C_{\mathscr A}f&=-\hbar\Delta_{\mathscr G}f+2\langle\mathscr A|df\rangle_{\mathscr G},\\
\mathscr A&=d\Phi+\mathscr A^\perp.
\end{align}
For a stationary density $\varpi$, the choice $\Phi=-(\hbar/2)\log\varpi$ yields:
\begin{align}
\operatorname{div}_{\mathscr G}\left(\varpi(\mathscr A^\perp)^\sharp\right)=0,
\end{align}
under suitable boundary conditions. The residual field describes stationary circulation, a nonreversible gauge field that preserves the invariant density while changing relaxation \citep{ohzeki2026_nonreversibleGaugeFields}. The associated twisted differential, \citep{novikov1981_multivaluedMorse, novikov1982_hamiltonianMorse}, satisfies:
\begin{align}
d_{\mathscr A,\hbar} &=\hbar d+\mathscr A\wedge,\\
d_{\mathscr A,\hbar}^2 &=\hbar\,d\mathscr A\wedge.
\end{align}
Its square thus records the exterior curvature of the transport one-form.
Girsanov's theorem, \citep{girsanov1960_measureTransform}, describes the transport deformation through a change of probability measure on diffusion paths. For a trajectory $r_t\in\mathscr R$, let $x_t\in\mathbb R^d$ denote its local coordinates. The reference diffusion satisfies an Itō equation analogous to the differential-equation~\eqref{eq:wng-general-transport}:
\begin{align}
dx_t&=b_0(x_t)\,dt+\mathbbm e(x_t)\,dW_t,\\
\mathbbm e\mathbbm e^\top&=2\hbar\mathscr G^{-1},
\end{align}
where $b_0$ is the geometric It\^o drift and $W_t$ is standard $d$-dimensional Brownian motion. The one-form $\mathscr A$ modifies the drift according to:
\begin{align}
b_{\mathscr A}&=b_0-2\mathscr A^\sharp,\\
v_t&=\mathbbm e(x_t)^{-1}(b_{\mathscr A}-b_0)(x_t).
\end{align}
The corresponding path probability measures $\mathbb P^0$ and $\mathbb P^{\mathscr A}$ are related by
\begin{align}\label{eq:wng-girsanov}
\frac{d\mathbb P^{\mathscr A}}{d\mathbb P^0}
\bigg|_{[0,T]}&=\exp\left( \int_0^T v_t\cdot dW_t -\frac12\int_0^T|v_t|^2dt \right).
\end{align}
Under Novikov's condition, \citep{novikov1972_stochasticIdentity}, this defines a valid change of measure. The resulting diffusion has drift $b_{\mathscr A}$ and retains the covariance $2\hbar\mathscr G^{-1}$.
The finite counterpart is DMTN, here we introduce a skew-symmetric feature form $A^\top=-A$. For equal-width feature factors, one possible parameterization is $A=W_-W_+^\top-W_+W_-^\top$. The directed in-context score and its Markov normalization, an imaginary gauge field in the sense of Hatano and Nelson \citep{hatano1996_nonHermitian}, are:
\begin{align}\label{eq:wng-incontext-antisymmetric}
\mathcal A&=\mathcal R A\mathcal R^\top,\\
\mathsf J&=\mathsf K\odot\exp_\odot(\mathcal A),\\
\mathsf P^{\mathrm{DMTN}}
&=\mathrm{diag}(\mathsf J\mathbf1)^{-1}\mathsf J.
\end{align}
An exact edge score $\mathcal A=\mathbf1\phi^\top-\phi\mathbf1^\top$ reproduces scalar destination reweighting. General directed scores admit circulation, measured by the finite cycle affinity, \citep{schnakenberg1976_networkTheory},
\begin{align}\label{eq:wng-cycle-affinity}
\mathfrak a_{ijk}
&=\log\frac{
\mathsf P_{ij}\mathsf P_{jk}\mathsf P_{ki}
}{
\mathsf P_{ji}\mathsf P_{kj}\mathsf P_{ik}
}\\
&=2\left(
\mathcal A_{ij}+\mathcal A_{jk}+\mathcal A_{ki}
\right),
\end{align}
where $\mathsf P=\mathsf P^{\mathrm{DMTN}}$. A nonzero affinity signals irreversible circulation and the violation of detailed balance, by Kolmogorov's criterion \citep{kolmogorov1936_markovChains}. These identities hold exactly at finite $N$. In a regular local continuum limit, sufficiently mild directed scores contribute a first-order drift while preserving the leading diffusion tensor $D^{ab}=\hbar\mathscr G^{ab}$.

The Cauchy--Green and Witten deformations also meet in conditioned feature maps used by diffusion transformers \citep{perez2018_film,peebles2022dit}. Consider an idealized affine pre-attention deformations:
\begin{align}\label{eq:adaln-map-cg}
\mathcal R_{\gamma,b}&=\mathcal R\,\mathrm{diag}(\gamma)+\mathbf1b^\top.
\end{align}
For a bilinear score $\mathsf S=\mathcal R W\mathcal R^\top$, the row-dependent additive terms cancel under softmax, leaving
\begin{align}\label{eq:adaln-triangular}
\mathsf S_{\gamma,b}
&\equiv
\mathcal R\,\mathrm{diag}(\gamma)
W\,\mathrm{diag}(\gamma)\mathcal R^\top
+\mathbf1\phi^\top,\\
\phi &=\mathcal R\,\mathrm{diag}(\gamma)W^\top b,
\end{align}
where $\equiv$ denotes equality modulo row-only terms. Feature scaling induces a congruence of the score form, affecting both its symmetric and antisymmetric components. The shift produces an exact destination potential and hence a Witten-type scalar reweighting, while preserving the existing cycle affinities. For a pure squared-distance DMAP bond, a common shift cancels identically from the pairwise separations. These relations describe the idealized pre-attention score, with normalization and residual pathways contributing additional transformations in a complete Transformer.
Together, these constructions distinguish the learned metric, reversible scalar transport, and irreversible directed transport within a common Markov framework.
\section{Bakry--\'Emery I: Geometry Without Coordinates}\label{sec:bakry-emery-one}

The preceding sections established how a learned feature-space quadratic form induces a Riemannian metric, while scalar potentials and directed fields determine transport upon it. We now consider the inverse problem of recovering this geometry directly from the Markov operator, using its action on observable functions. Throughout, the manifold is compact and smooth, the stationary density is smooth and positive, and the generators have discrete spectrum.
Recall the random-walk Laplacian $\mathsf L=\mathsf I-\mathsf P$, for observables $f,g$ on the context, define the pointwise \textit{carr\'e-du-champ} \cite{bakry2014_markovDiffusionOperators, berry2020_spectralExteriorCalculus, jones2024_diffusionGeometry},
\begin{align}\label{eq:bei-finite-gamma}
\Gamma_{\mathsf L,i}(f,g)&=\frac12\sum_j\mathsf P_{ij}(f_j-f_i)(g_j-g_i).
\end{align}
This is half the conditional mean product of one-step increments at context $i$. The construction depends only on this operator, and its stationary average recovers the Dirichlet energy,
\begin{align}\label{eq:bei-finite-energy}
\sum_i\pi_i\Gamma_{\mathsf L,i}(f,f)&=\langle f|\mathsf L|f\rangle_\pi.
\end{align}
For a continuum generator $\mathscr C$, with our positive-sign convention, the corresponding pairing is defined through its action on products,
\begin{align}\label{eq:bei-gamma-definition}
\Gamma_{\mathscr C}(f,g)&=\frac12\left(f\mathscr Cg+g\mathscr Cf-\mathscr C(fg)\right).
\end{align}
Under the continuum scaling $\mathscr C_\beta=4\hbar\beta(\mathsf I-\mathsf P_\beta)$, the pairing becomes $4\hbar\beta\,\Gamma_{\mathsf L}$.
For the local generator obtained in Witten--Novikov--Girsanov,
\begin{align}\label{eq:bei-local-generator}
\mathscr C_{\mathscr A}f
&=-\hbar\Delta_{\mathscr G}f
+2\langle\mathscr A|df\rangle_{\mathscr G},
\end{align}
the Leibniz rule cancels the first-order transport terms, leaving our diffusion tensor:
\begin{align}\label{eq:bei-cometric}
\Gamma_{\mathscr C_{\mathscr A}}(f,g)
&=\hbar\langle df|dg\rangle_{\mathscr G}
=\hbar\,\mathscr G^{-1}(df,dg).
\end{align}

Diffusion harmonics, \cite{coifman2006_geometricHarmonics}, provide a spectral realization geometry. For the reversible Witten operator, an orthonormal eigenfunction family $\mathscr C_{\mathrm W}\phi_k=\lambda_k\phi_k$ in the stationary Hilbert space. Their differential pairings satisfy:
\begin{align}\label{eq:bei-harmonic-gamma}
\Gamma(\phi_j,\phi_k)
&=\frac12\left(
(\lambda_j+\lambda_k)\phi_j\phi_k
-\mathscr C_{\mathrm W}(\phi_j\phi_k)
\right).
\end{align}
Where sufficiently many differentials $d\phi_k$ span the cotangent space, these pairings determine the co-metric. A truncated harmonic family provides a finite-dimensional approximation to this differential geometry without the need of a chart\footnote{The induced geodesic distance can likewise be recovered through observables with bounded gradient energy,
\begin{align}\label{eq:bei-intrinsic-distance}
d_{\mathscr G}(x,y)&=\sup_{\Gamma(f,f)\le\hbar}|f(x)-f(y)|,
\end{align}
with the bound imposed pointwise over Lipschitz functions. The finite pairing supplies an analogous energy-based construction, whose geodesic interpretation emerges in the continuum limit.}. Only in the continuum limit, infinitely many harmonics, are points arbitrarily resolved, \citep{pesenson2004_poincarePaleyWienerManifolds, pesenson2008_samplingPaleyWienerGraphs}.

The same operator calculus also constructs the Hessian. For a local diffusion,
\begin{align}\label{eq:bei-hessian-from-gamma}
2\hbar^2\mathrm{Hess}_{\mathscr G}f(\nabla g,\nabla h)
&=\Gamma(g,\Gamma(f,h))+\Gamma(h,\Gamma(f,g))-\Gamma(f,\Gamma(g,h)).
\end{align}

Curvature enters through the iterated \textit{carr\'e-du-champ},
\begin{align}\label{eq:bei-gamma-two}
\Gamma_2(f)&=\Gamma(f,\mathscr Cf)
-\frac12\mathscr C\Gamma(f,f).
\end{align}
For Witten's reversible diffusion, the Bochner identity gives:
\begin{align}\label{eq:bei-bochner}
\Gamma_2(f)
&=\hbar^2|\mathrm{Hess}_{\mathscr G}f|^2
+\hbar^2\mathrm{Ric}_{\mathscr F}(\nabla f,\nabla f),\\
\mathrm{Ric}_{\mathscr F}
&=\mathrm{Ric}_{\mathscr G}
+\frac{2}{\hbar}\mathrm{Hess}_{\mathscr G}\mathscr F.
\end{align}
The scalar potential contributes to the weighted Bakry--\'Emery Ricci tensor through its Hessian. Lower curvature bounds can be formulated intrinsically by the curvature-dimension inequality
\begin{align}
\Gamma_2(f)&\geq\kappa\,\Gamma(f,f),
\end{align}
which governs gradient estimates and contraction properties.
For a general transport one-form $\mathscr A$, the corresponding Bochner identity becomes
\begin{align}\label{eq:bei-general-bochner}
\Gamma_2(f)
&=\hbar^2|\mathrm{Hess}_{\mathscr G}f|^2
+\hbar^2\mathrm{Ric}_{\mathscr G}(\nabla f,\nabla f)
+2\hbar(\nabla\mathscr A)_{\mathrm{sym}}(\nabla f,\nabla f).
\end{align}
The quadratic contraction selects the symmetric covariant derivative of the transport field. Its antisymmetric derivative defines the curvature two-form $\mathscr B=d\mathscr A$, whose flux through a surface $\Sigma$ determines the circulation $\oint_{\partial\Sigma}\mathscr A=\int_\Sigma\mathscr B$.
Let $d\pi=\varpi\,dV_{\mathscr G}$ denote the stationary measure, with effective potential $\Phi=-(\hbar/2)\log\varpi$. The associated stationary probability current satisfies:
\begin{align}\label{eq:bei-stationary-current}
J&=-2\varpi(\mathscr A-d\Phi)^\sharp,\\
\operatorname{div}_{\mathscr G}J&=0.
\end{align}
For the reversible Witten field $\mathscr A=d\mathscr F$, the equilibrium choice $\Phi=\mathscr F$ gives $J=0$. A general transport field can sustain stationary circulation, measured by the antisymmetric operator pairing:
\begin{align}\label{eq:bei-omega-definition}
\Omega_\pi(f,g)
&=\frac12\left(
\langle\mathscr Cf|g\rangle_\pi
-\langle f|\mathscr Cg\rangle_\pi
\right).
\end{align}
The same construction applies at finite $N$,
\begin{align}\label{eq:bei-finite-omega}
\Omega_\pi^{\mathsf L}(f,g)
&=\frac12\left(
\langle\mathsf Lf|g\rangle_\pi
-\langle f|\mathsf Lg\rangle_\pi
\right),\\
\Omega_\pi^{\mathscr C_\beta}
&=4\hbar\beta\,\Omega_\pi^{\mathsf L}.
\end{align}
In the continuum, integration by parts yields:
\begin{align}\label{eq:bei-omega-current}
\Omega_\pi(f,g)
&=\int_{\mathscr R}f\,dg(J)\,dV_{\mathscr G}.
\end{align}

The current also determines the local transport curvature. Lowering its index with the learned metric gives:
\begin{align}\label{eq:bei-current-curvature}
\mathscr B
&=-\frac12d\left(\varpi^{-1}J^\flat\right).
\end{align}
Global circulation can additionally occur through closed one-forms with nontrivial periods.\footnote{On a circle, the field $\mathscr A=a\,d\theta$ has $d\mathscr A=0$ and circulation $\oint\mathscr A=2\pi a$. The stationary current therefore carries global information beyond the local curvature two-form.}

The Markov operator thus supplies an intrinsic geometric calculus. The pairing $\Gamma$ determines the local co-metric and differential geometry, $\Gamma_2$ characterizes weighted curvature, and $\Omega_\pi$ isolates stationary circulation. Their finite definitions follow directly from the transition operator, while the local continuum limit recovers the familiar Riemannian tensors.
\section{Beurling--Deny: Nonlocal Diffusion}\label{sec:beurling-deny}

The preceding geometric construction relied on the locality of diffusion, where transitions concentrate around infinitesimal neighborhoods and yield a Riemannian differential operator. We now let the learned quadratic form become indefinite, which keeps the finite Markov operator reversible but lets distant configurations compete for affinity and, in the continuum, admit nonlocal transitions.
Consider the difference of two PSD feature forms, $G=G_+-G_-$, this is the Krein-metric\footnote{Analogous to the Mahalanobis metric.} and subsequent Krein-bond\footnote{Analogous to the Gram metric}, $\mathcal G=\mathcal R G\mathcal R^\top$,
\begin{align}\label{eq:bd-indefinite-bond}
\mathcal H&=n\mathbf1^\top+\mathbf1n^\top-2\mathcal G,
\end{align}
The bond is the squared separation of sample pairs under $W_+$ minus that under $W_-$: the positive sector suppresses feature separations, while the negative sector enhances their relative affinity. This competition has a Hebbian interpretation, \citep{pehlevan2015_hebbianAntiHebbian, pehlevan2018_whySimilarityMatching, abid2018_contrastivePCA}, in which learned associations are reinforced against alternative configurations. For observed transition counts $\mathsf T$, maximizing the log-likelihood $\langle\mathsf T,\log_\odot\mathsf P\rangle$ produces the contrastive update
\begin{align}\label{eq:bd-hebbian-gradient}
\nabla_G\langle\mathsf T,\log_\odot\mathsf P\rangle
&=\beta\,\mathcal R^\top\left(
\mathrm{Lap}\left(\mathrm{diag}(\mathsf T\mathbf1)\mathsf P\right)
-\mathrm{Lap}(\mathsf T)
\right)\mathcal R,\\
\mathrm{Lap}(W)&=\mathrm{diag}\left((W+W^\top)\mathbf1\right)-W-W^\top.
\end{align}
The update compares the observed transitions $\mathsf T$ with the competing model flux $\mathrm{diag}(\mathsf T\mathbf1)\mathsf P$ generated by the current kernel; a single observed pair corresponds to $\mathsf T=e_ie_j^\top$.
The kernel $\mathsf K=\exp_\odot(-\beta\mathcal H)$ remains symmetric and strictly positive. Its row normalization and stationary law are:
\begin{align}\label{eq:bd-atap-markov}
\mathsf P&=\mathrm{diag}(\mathsf K\mathbf1)^{-1}\mathsf K,\\
\pi&=\frac{\mathsf K\mathbf1}{\mathbf1^\top\mathsf K\mathbf1},
\end{align}
and satisfy detailed balance, $\mathrm{diag}(\pi)\mathsf P=\mathsf P^\top\mathrm{diag}(\pi)$. Consequently, the Dirichlet energy:
\begin{align}\label{eq:bd-finite-energy}
\mathcal E(f,f)
&=\langle f|\mathsf I-\mathsf P|f\rangle_\pi
=\frac{f^\top\left(\mathrm{diag}(\mathsf K\mathbf1)-\mathsf K\right)f}{\mathbf1^\top\mathsf K\mathbf1}\geq0
\end{align}
remains nonnegative even though the feature form is indefinite. We call this enlargement ATAP.

The geometric consequence appears in the distinction between tangent vectors and finite chords. The induced metric
\begin{align}\label{eq:bd-local-pullback}
\mathscr G(x)&=J_\iota(x)^\top GJ_\iota(x)
\end{align}
can remain PSD on every tangent space of $\mathscr R$, preserving its local Riemannian geometry. Meanwhile, a chord joining distant points may have a small or vanishing indefinite bond $\mathcal H(x,y)$. Such pairs compete with nearby transitions in the normalized kernel, with a weight set by the bond minima and their transverse Hessians: negative remote costs produce dominant transitions, while competing minima balanced as in Eq.~\eqref{eq:bd-critical-balance} support finite-rate connections.
To determine the continuum behavior, consider the normalized kernel $\mathscr P_\beta$ and its positive generator,
\begin{align}\label{eq:bd-continuum-kernel}
\mathscr P_\beta f(x)
&=\int_{\mathscr R}\mathscr P_\beta(x,y)f(y)\,d\nu(y),\\
\mathscr C_\beta f(x)
&=4\hbar\beta\int_{\mathscr R}
\mathscr P_\beta(x,y)
\left(f(x)-f(y)\right)\,d\nu(y).
\end{align}
In the simultaneous dense-sampling and shrinking-bandwidth limit, we separate transitions into near and far contributions using a fixed background distance $d_0(x,y)$ and a separation $\delta>0$. Near-diagonal transitions admit the Einstein expansion, while remote transitions retain their finite displacements and are governed by their probability per unit rescaled diffusion time.
The decisive condition for a surviving jump sector is the weak convergence, away from the diagonal,
\begin{align}\label{eq:bd-survival-condition}
4\hbar\beta\,\mathscr P_\beta^{\mathrm{far}}(x,dy)
&\rightharpoonup
\mathscr K_{\mathrm{jump}}(x,dy),
\end{align}
where $\mathscr P_\beta^{\mathrm{far}}$ is the transition measure restricted to $d_0(x,y)\geq\delta$. The limiting jump-rate measure can include isolated remote partners. Far-transition probabilities of order $o(\beta^{-1})$ vanish on this time scale, those of order $\beta^{-1}$ yield finite jump rates, and stronger channels evolve on faster scales.

A Laplace approximation makes this distinction quantitative. Suppose the diagonal minimum contributes mass of order $\beta^{-d/2}$, while a remote critical set of codimension $q$ contributes $\beta^{-q/2}e^{-\beta h_\beta}$, with excess bond cost $h_\beta$. Assuming controlled nondegenerate transverse Hessians and bounded positive prefactors, the rescaled remote rate behaves as
\begin{align}\label{eq:bd-laplace-balance}
\lambda_\beta
&\sim\beta^{1+(d-q)/2}e^{-\beta h_\beta}.
\end{align}
A finite nonzero rate follows from the balance
\begin{align}\label{eq:bd-critical-balance}
\beta h_\beta
-\left(1+\frac{d-q}{2}\right)\log\beta
&\longrightarrow c.
\end{align}
For isolated remote minima, $q=d$, this reduces to $\beta h_\beta-\log\beta\to c$. For a fixed learned form and embedding, an isolated remote minimum has a fixed excess cost, so the critical regime requires a bandwidth-dependent family $G_\beta$ or another parameter-dependent bond landscape. Exactly degenerate competing minima typically induce rapid coupling, while a fixed positive excess cost suppresses the remote channel exponentially.

When the local and finite-rate remote contributions both survive, their common microscopic origin yields the continuum generator:
\begin{align}\label{eq:bd-jump-generator}
\mathscr C f(x)
&=\mathscr C_{\mathrm{local}}f(x)
+\int_{\mathscr R}
\left(f(x)-f(y)\right)\,
\mathscr K_{\mathrm{jump}}(x,dy).
\end{align}
Here $\mathscr C_{\mathrm{local}}$ is the reversible local diffusion operator on the learned Riemannian geometry. Let $d\mu(x)=\pi(x)d\nu(x)$ denote the stationary measure. Finite ATAP is exactly reversible, and this reversibility passes to the limit provided the stationary measures and Dirichlet forms converge; it then requires the stationary jump measure
\begin{align}
d\mathscr J_s(x,y)
&=d\mu(x)\,\mathscr K_{\mathrm{jump}}(x,dy)
\end{align}
to be symmetric under exchange of $x$ and $y$.
The Beurling--Deny decomposition, \citep{beurlingDeny1958_dirichletSpaces, lejan1978_mesuresDirichlet, fukushima2011_dirichletForms, ma1992_dirichletForms, albeverio1995_dirichletOperatorsStochasticAnalysis},  then gives the associated Dirichlet energy,
\begin{align}\label{eq:bd-energy}
\mathscr E(f,f)
&=\mathscr E_{\mathrm{local}}(f,f)
+\frac12\int_{\mathscr R\times\mathscr R}
\left(f(y)-f(x)\right)^2
\,d\mathscr J_s(x,y).
\end{align}
Since the Markov dynamics conserve probability, the killing term of the general decomposition vanishes, leaving local gradient fluctuations and finite-distance transitions as the two contributions.
This structure extends the exterior differential, \citep{cipriani2003_derivationsSquareRootsDirichletForms}, of the finite Markov walk: the local differential $df$ is supplemented by differences between endpoint pairs,
\begin{align}\label{eq:bd-nonlocal-supercharge}
(\mathscr Q_{\mathrm{jump}}f)(x,y)
&=f(y)-f(x),\\
\mathscr Qf&=(df,\mathscr Q_{\mathrm{jump}}f),\\
\mathscr D_s&=\mathscr Q^\dagger\mathscr Q.
\end{align}
The local cotangent and endpoint-pair spaces carry the weights of $\mathscr E$, with the adjoint taken in their corresponding Hilbert spaces. The graded Dirac operator constructed from $\mathscr Q$ and $\mathscr Q^\dagger$ squares to the positive symmetric generator on observables.

Witten's scalar deformation also extends to endpoint pairs. For a symmetric reference bond $\mathscr K_0$, the weighted differential takes the form
\begin{align}\label{eq:bd-half-link}
(\mathscr Q_{\mathscr F,\hbar}f)(x,y)&=\hbar\sqrt{\mathscr K_0(x,y)}
\left(
e^{(\mathscr F(y)-\mathscr F(x))/(2\hbar)}f(y)
-e^{(\mathscr F(x)-\mathscr F(y))/(2\hbar)}f(x)
\right).
\end{align}
It annihilates the positive field $e^{-\mathscr F/\hbar}$, so, as in the local Witten construction, its weighted square admits a ground-state transform to a conservative reversible Markov operator.

Consider two well-separated manifold sheets with competing low-cost chords joining corresponding points. According to whether the excess cost of these chords is fixed, critically balanced, or exactly degenerate, the associations disappear, survive as finite-rate jumps, or produce rapid dynamical identification of the sheets, each regime arising from a symmetric kernel with reversible Markov dynamics.\footnote{An explicit realization embeds two circles in $\mathbb R^4$ as $\iota_s(\theta)=(\cos\theta,\sin\theta,s,s)$, $s\in\{0,1\}$, with $G_\beta=\mathrm{diag}(1,1,1,h_\beta-1)$. The bond becomes $\mathcal H_\beta=4\sin^2\left((\theta-\theta')/2\right)+h_\beta(s-s')^2$: each circle keeps its intrinsic metric $d\theta^2$, while corresponding points, at ambient distance $\sqrt2$, cost only $h_\beta$. Choosing $h_\beta=(\log\beta+c)/\beta$ gives a cross-sheet transition probability $e^{-c}/\beta$ under uniform sampling, and the rescaled generator converges formally to $\mathscr Cf(s,\theta)=-\hbar\partial_\theta^2f(s,\theta)+4\hbar e^{-c}\bigl[f(s,\theta)-f(1-s,\theta)\bigr]$, a reversible local diffusion with a finite-rate jump derived directly from an indefinite learned form.}
Finally, the connection to bilinear attention follows from the Gram representation of Eq.~\eqref{eq:bd-indefinite-bond}, whose entrywise exponential satisfies
\begin{align}\label{eq:bd-bilinear-site-factor}
\exp_\odot(\beta\mathcal G)
&=\mathrm{diag}\!\left(e^{\beta n/2}\right)
\exp_\odot\!\left(-\frac{\beta}{2}\mathcal H\right)
\mathrm{diag}\!\left(e^{\beta n/2}\right).
\end{align}
A symmetric bilinear score thus consists of a learned Krein bond together with endpoint factors of the Witten type, and its row-normalized operator remains reversible.

The indefinite feature form therefore enlarges the range of learned symmetric associations while preserving positive Markov energy and detailed balance. Under the critical rate scaling, these associations generate transitions across finite geometric distances, and the Beurling--Deny decomposition provides their corresponding diffusion, energy, and differential calculus.
\section{Courr\`ege: Continuum Attention}\label{sec:courrege-attention}

We have now relaxed two restrictions of Einstein's original diffusion. Witten and Novikov--Girsanov changed transport without requiring a change of the local diffusion tensor, while Beurling--Deny allowed preferred bonds to connect samples that remain far apart. Both are already present in a query--key score form of attention. The question is how to recover them from a full head, and what continuum operator survives when both the bonds and their direction are learned.

Take the in-context samples as the rows $R_i\in\mathbb R^D$ of $\mathcal R\in\mathbb R^{N\times D}$. For projections $W_Q,W_K$, introduce the unnormalized head and its row-normalized Markov operator,
\begin{align}\label{eq:ca-attention-kernel}
  \mathsf J&=\exp_\odot(\beta\mathcal R M\mathcal R^\top),\\
  M&=W_QW_K^\top,\\
  \mathsf P&=\mathrm{diag}(\mathsf J\mathbf1)^{-1}\mathsf J.
\end{align}
The score need not be symmetric, but every entry of $\mathsf J$ is positive. Split the feature pairing into the symmetric Krein bond and the antisymmetric field,
\begin{align}\label{eq:ca-feature-split}
  G&=M+M^\top,\\
  A&=M-M^\top.
\end{align}
For the forward head and the head with query and key exchanged, the decomposition is exact: their elementwise product removes the odd field and their ratio removes the even bond,
\begin{align}\label{eq:ca-product-ratio}
  \mathsf J\odot\mathsf J^\top
    &=\exp_\odot(\beta\mathcal R G\mathcal R^\top),\\
  \mathsf J\oslash\mathsf J^\top
    &=\exp_\odot(\beta\mathcal R A\mathcal R^\top).
\end{align}
Here the transpose exchanges the two ends of the \emph{unnormalized} score; the time reversal of a stationary Markov process comes later. The head is reconstructed by the elementwise geometric means,
\begin{align}\label{eq:ca-recombination}
  \mathsf K_{\mathrm{rev}}&=(\mathsf J\odot\mathsf J^\top)^{\odot1/2},\\
  \mathsf U&=(\mathsf J\oslash\mathsf J^\top)^{\odot1/2},\\
  \mathsf J&=\mathsf K_{\mathrm{rev}}\odot\mathsf U.
\end{align}

This brings us back to the Gaussian bond of DMAP and ATAP. Let $n=\mathrm{diag}(\mathcal{R}G\mathcal{R}^\top)$ and let $\mathcal H_G$ denote the squared Krein bond cost from Beurling--Deny. Completing the square in the symmetric score gives
\begin{align}\label{eq:ca-bond-return}
  \mathsf K_{\mathrm{rev}}
  &=\mathrm{diag}(e^{\beta n/4})
    \exp_\odot(-\beta\mathcal H_G/4)
    \mathrm{diag}(e^{\beta n/4}).
\end{align}
Row normalization removes the query-only factor, but the key-side scalar remains; it is precisely the sort of exact Witten deformation encountered in DMAP. Define the directed link $\mathcal{A}=\tfrac\beta2 \mathcal{R}A\mathcal{R}^\top$. The same head can be written as a bond, a key site, and a directed field,
\begin{align}\label{eq:ca-three-channels}
\mathsf P&=\mathrm{rownorm}\!\left(
    \bigl(\exp_\odot(-\beta\mathcal H_G/4)\,
    \mathrm{diag}(e^{\beta n/4})\bigr)
    \odot\exp_\odot(\mathcal A)\right).
\end{align}
The ``square'' of a head therefore returns the reversible bond together with its induced scalar weights. This $\mathsf K_{\mathrm{rev}}$ is symmetric and entrywise positive even when the learned bond is indefinite and the matrix itself is not PSD. Set $W_\pm=(W_Q\pm W_K)/\sqrt2$, such that
\begin{align}\label{eq:ca-sheets}
  G&=W_+W_+^\top-W_-W_-^\top,\\
  A&=W_-W_+^\top-W_+W_-^\top.
\end{align}
Coincident query and key projections have $W_-=0$, hence a definite DMAP bond and only an exact scalar site. Mirror projections, with $M$ symmetric but indefinite, give an indefinite yet reversible ATAP bond. Releasing the link while keeping the definite bond gives DMTN; general relative positions recover full ATTN. These are two independent switches: \emph{definiteness of the bond} and \emph{exactness of the link}. The link also has a directly measurable finite curvature: the log-ratio of forward and backward probabilities around a triangle is
\begin{align}\label{eq:ca-triangle-affinity}
  \log\frac{\mathsf P_{ij}\mathsf P_{jk}\mathsf P_{ki}}{\mathsf P_{ji}\mathsf P_{kj}\mathsf P_{ik}}
  &=\beta(R_j-R_i)A(R_k-R_i)^\top.
\end{align}
The row normalizations and site factors cancel around the loop. This is the finite symplectic-area counterpart of the transport curvature $\mathscr{B}=d\mathscr A$, and even a kernel with a perfectly definite bond may carry a nonzero affinity.

What does the full head become as the sampling is refined? Replace the finite sum by an integral against $d\nu(y)=\rho(y)dV_{\mathscr G}(y)$ wherever a local metric $\mathscr G$ exists. The same normalized two-point kernel has near-diagonal and remote sectors. Under the diffusive scaling of Einstein, a controlled near sector produces the learned Laplace--Beltrami operator and a general first-order transport field, while a surviving remote sector produces directed jump rates. In the finite-activity regime the positive continuum generator takes the Courr\`ege form
\begin{align}\label{eq:ca-courrege-generator}
\mathscr C f(x)&=-\hbar\Delta_{\mathscr G}f(x)+2\langle\mathscr A | df\rangle_{\mathscr G}(x)+\int_{\mathscr R}d\nu(y)\,\mathscr K_{\mathrm{jump}}(x,y)(f(x)-f(y)).
\end{align}
The jump rate $\mathscr K_{\mathrm{jump}}(x,y)$ need not be symmetric, and may be distributional at preferred remote partners. Courr\`ege's more general theory allows arbitrary PSD diffusion matrices and compensated singular small jumps; here we retain the limit appropriate to our learned local metric and finite remote rates. The local part is the Novikov--Girsanov operator, the symmetric remote part extends Beurling--Deny, and all three terms come from one microscopic head. The learned-Krein jump limit requires the critical rate scaling of Section~\ref{sec:beurling-deny}.

To distinguish dissipation from irreversibility we must now use the \emph{stationary} probability law. At finite $N$, let $\pi^\top\mathsf P=\pi^\top$, and define the stationary flux and its reversal,
\begin{align}\label{eq:ca-stationary-flux}
  \mathsf A &=\mathrm{diag}(\pi)\mathsf{P},\\
  \mathsf P^*&=\mathrm{diag}(\pi)^{-1}\mathsf P^\top\mathrm{diag}(\pi).
\end{align}
The adjoint $\mathsf P^*$ is genuine Markov time reversal. The product and ratio of $\mathsf A$ and $\mathsf A^\top$ produce the stationary geometric mean and the affinity,
\begin{align}\label{eq:ca-flux-variables}
\bar{\mathsf A}&=(\mathsf A\odot\mathsf A^\top)^{\odot1/2},\\
\theta &=\tfrac12\log_\odot(\mathsf A\oslash\mathsf A^\top).
\end{align}
These differ from the raw-score product and ratio only by site terms involving $\pi$ and the row sums $Z=\mathsf J\mathbf1$. In particular $\theta=\mathcal{A}+(u\mathbf1^\top-\mathbf1u^\top)/2$ with $u=\log\pi-\log Z$, so every cycle affinity is preserved. The stationary symmetric conductance and directed current are then
\begin{align}\label{eq:ca-cosh-sinh}
Q &=\bar{\mathsf A}\odot\cosh_\odot\theta,\\
\mathsf j &=2\bar{\mathsf A}\odot\sinh_\odot\theta.
\end{align}
This is the second recombination. The geometric mean $\bar{\mathsf A}$ removes the directional tilt and recovers an ATAP-type bond dressed by sites, while the actual symmetric energy uses $Q$, enhanced by the even part $\cosh\theta$ of the field. Deleting the antisymmetric logits therefore does not in general give the symmetric part of the stationary generator. The ratio also shows why reversing raw scores differs from reversing stationary trajectories: the invariant masses and row sums contribute an exact site correction, whose contributions disappear on closed cycles.
Let $\mathsf L=\mathsf I-\mathsf P$, in the stationary inner product it splits into self-adjoint dissipation and skew-adjoint circulation,
\begin{align}\label{eq:ca-operator-split}
  \mathsf L&=\mathsf L_s+\mathsf N,\\
  \mathsf L_s&=\frac12(\mathsf L+\mathsf L^*),\\
  \mathsf N&=\frac12(\mathsf L-\mathsf L^*).
\end{align}
The symmetric part is again a positive Kirchhoff operator, with conductances $Q_{ij}$,
\begin{align}\label{eq:ca-flux-energy}
  \langle f|\mathsf{L}|f\rangle_\pi
  &=\frac12\sum_{i,j}Q_{ij}(f_i-f_j)^2\ge0.
\end{align}
Neither the negative feature sector nor directed transport compromises this positivity: directed transport changes decay and circulation without destabilizing the positive Markov generator. For a nonzero eigenmode $\mathsf L\phi=\lambda\phi$, the real part of its eigenvalue is its dissipative energy,
\begin{align}\label{eq:ca-spectral-dissipation}
  \mathrm{Re}\,\lambda
  &=\frac{\langle\phi|\mathsf L_s|\phi\rangle_\pi}
          {\langle\phi|\phi\rangle_\pi}\ge0,
\end{align}
while the imaginary part is controlled by $\mathsf N$, so non-real spectra are possible.

Witten and symmetric Beurling--Deny operators could be written as a single square $\mathscr Q^\dagger\mathscr Q$; an irreversible operator generally cannot. Forward and reversed stationary fluxes instead give two twisted edge differentials,
\begin{align}\label{eq:ca-twisted-differentials}
  (\mathsf D_Kf)_{ij}&=\bar{\mathsf A}_{ij}^{1/2}\left(e^{-\theta_{ij}/2}f_i-e^{\theta_{ij}/2}f_j\right),\\
  (\mathsf D_Qg)_{ij}&=\bar{\mathsf A}_{ij}^{1/2}\left(e^{\theta_{ij}/2}g_i-e^{-\theta_{ij}/2}g_j\right),
\end{align}
with the edge inner product $\langle w|v\rangle_E=\tfrac12\sum_{i,j}w_{ij}v_{ij}$ over ordered pairs and adjoints taken relative to $\langle\cdot|\cdot\rangle_\pi$ on sites. In the finite theory they obey the exact identity
\begin{align}\label{eq:ca-two-roots}
  \mathsf L&=\mathsf D_Q^\dagger\mathsf D_K+\mathsf V,\\
  \mathsf V_{ii}
  &=\pi_i^{-1}\sum_{j\ne i}\bar{\mathsf A}_{ij}
         (\cosh\theta_{ij}-1).
\end{align}
The twisted differentials do not annihilate constants when $\theta\neq0$, and $\mathsf V$ records exactly this mismatch, since stationarity gives $\pi_i=\sum_j\bar{\mathsf A}_{ij}\cosh\theta_{ij}$. The field energy $\mathsf V$ is nonnegative and vanishes with detailed balance, where both differentials reduce to the ordinary weighted difference. Exchange of the two roots sends $\theta\to-\theta$ and reverses the stationary dynamics, so the one-root reversible theory is the special case in which the forward and backward branches coincide. In the continuum local sector the corresponding twisted differentials carry opposite transport fields, whose curvatures are $\pm\mathscr{B}$.

We have recovered DMAP's learned local geometry, ATAP's remote bonds, and DMTN's circulating field inside one attention generator. But the geometry is no longer entirely a Riemannian tensor. What becomes of the \textit{carr\'e-du-champ}, intrinsic distance, and $\Gamma_2$ when the transitions themselves are nonlocal?
\section{Bakry--\'Emery II: What Remains of Geometry?}\label{sec:bakry-emery-two}

For the full positive attention generator $\mathscr C$, the \textit{carr\'e-du-champ} is still determined by the operator's action on products, and its continuum local-plus-jump form is:
\begin{align}\label{eq:beii-general-gamma}
\Gamma_{\mathscr C}(f,g)(x)&=\hbar\,\mathscr G^{-1}(df,dg)(x)+\frac12\int_{\mathscr R}d\nu(y)\,\mathscr K_{\mathrm{jump}}(x,y)(f(y)-f(x))(g(y)-g(x)).
\end{align}
The first contribution measures infinitesimal displacements in the tangent metric; the second measures differences between distinct endpoints. The second term is not a quadratic form in $df$, so no Riemannian tensor reproduces it. The drift disappears by the product rule, while the directed jump rates remain in the pointwise fluctuations. Yet the finite operator already contained this term: at finite $N$ we had $\Gamma_i(f,g)=\frac12\sum_j\mathsf P_{ij}(f_j-f_i)(g_j-g_i)$, and for $j\ne i$ the value $\Gamma_i(\delta_j,\delta_j)=\mathsf P_{ij}/2$ reads off the one-step transition probability, whose rescaled value $4\hbar\beta\mathsf P_{ij}$ is the continuum jump rate. Because $\Gamma_j(\delta_i,\delta_i)=\mathsf P_{ji}/2$, the pointwise pairing retains the direction of each transition, while stationary circulation is isolated by the antisymmetric pairing $\Omega_\pi(f,g)=\frac12f^\top\mathsf j\,g$ of Courr\`ege. The operator calculus therefore survives even when the tensor description does not.

The price of nonlocality is that the ordinary diffusion chain rule fails. Nevertheless we may iterate the same construction,
\begin{align}\label{eq:beii-gamma-two}
\Gamma_{2,\mathscr C}(f)&=\Gamma_{\mathscr C}(f,\mathscr Cf)
    -\frac12\mathscr C\Gamma_{\mathscr C}(f,f).
\end{align}
On the finite graph this is an exact algebraic expression involving one- and two-step transitions, with no charts or limiting manifold in sight. In the local reversible limit Bochner's identity turns it into a Hessian square and the Bakry--\'Emery Ricci tensor. With remote jumps, local derivatives and endpoint differences mix, and no single Ricci tensor represents the result. Curvature inequalities such as $\Gamma_2(f)\geq\kappa\Gamma(f,f)$ can still be investigated, but their meaning is now a bound on the dynamics rather than a statement about the curvature of a smooth metric, as in the Bakry--\'Emery curvature of graphs \citep{sturm1994_localDirichletSpaces, lin2010_ricciCurvatureGraphs}, where the failure of the chain rule already forced a modified curvature-dimension condition \citep{bauer2015_liYauGraphs}.\footnote{Other discrete notions of curvature include Ollivier's coarse Ricci curvature, defined by the Wasserstein contraction of the walk and consistent with the Bakry--\'Emery Ricci curvature for diffusions with drift \citep{ollivier2009_ricciCurvatureMar}, and the entropic Ricci curvature of Erbar and Maas, defined by convexity of the entropy along a discrete transport metric \citep{erbar2012_ricciCurvatureMarkovChains}.}

There is a sharp physical way to see what has changed. For ordinary local heat diffusion, the short-time density $p_t^{\mathrm{loc}}(x,y)$ of the local semigroup $e^{-t\mathscr C_{\mathrm{loc}}}$, taken with respect to the Riemannian volume, remembers geodesic distance through Varadhan's asymptotic relation,
\begin{align}\label{eq:beii-varadhan}
  -4\hbar t\log p_t^{\mathrm{loc}}(x,y)&\longrightarrow d_{\mathscr G}(x,y)^2,
\end{align}
and any other smooth positive reference measure leaves the limit unchanged. But if a bounded, continuous, nonzero remote jump density connects the same separated endpoints, the short-time density of the full process behaves as $p_t(x,y)\sim t\,\mathscr K_{\mathrm{jump}}(x,y)$, rather than as an exponentially suppressed Gaussian tail. Consequently the same logarithmic scaling tends to zero along that channel. The geodesic inferred from local heat asymptotics is no longer a universal description of how the process can move. This is exactly the possibility left open by the Krein-bond scaling argument; whether trained heads realize the critical scaling is an empirical question.

A manifold and its metric may describe the infinitesimal sector of the data without exhausting the relationships learned by the operator. The full attention geometry is the geometry of transition strengths: local fluctuations and remote bonds through $\Gamma$, their curvature in the sense of curvature-dimension bounds through $\Gamma_2$, and directional circulation through $\Omega_\pi$. The harmonics remain useful probes, but no preferred coordinates or finite spectral truncation defines these objects. In this sense the cracks in the manifold hypothesis are not merely a failure of cartography. They tell us that the map itself may need to be an operator.
\section{Experiment: Swiss SiT}
\label{sec:exp-swiss}\label{sec:exp:swiss}

We begin with a simple generative experiment using a flow-based Transformer (SiT). Given a collection of samples, flow matching constructs a continuous path between the data distribution and Gaussian noise. A neural network learns the velocity field along this path, which can then be integrated backwards from noise to generate new samples. In our setting, each sample is a point on the Swiss roll, and a context $\mathcal R_\tau$ consists of $N$ such points evolving together. The Transformer therefore contains two distinct mechanisms of transport. The outer flow evolves the context through generative time $\tau$, while the attention operator allows information to move between the points within each context. A single-layer transformer should be sufficient to learn this flow. We test both orderings across DMAP, DMTN, ATAP, and ATTN. Although the ordering changes the learned attention operator, it need not change the quality of generation. This allows us to adopt the FFN$\to$MHSA convention developed earlier, in which the FFN first deforms the feature representation and MHSA subsequently measures its learned geometry and transports information over it.

\begin{figure}[t]
\centering
\begin{tikzpicture}
\node[anchor=south west, inner sep=0] (image) at (0,0) {%
    \includegraphics[width=\textwidth]{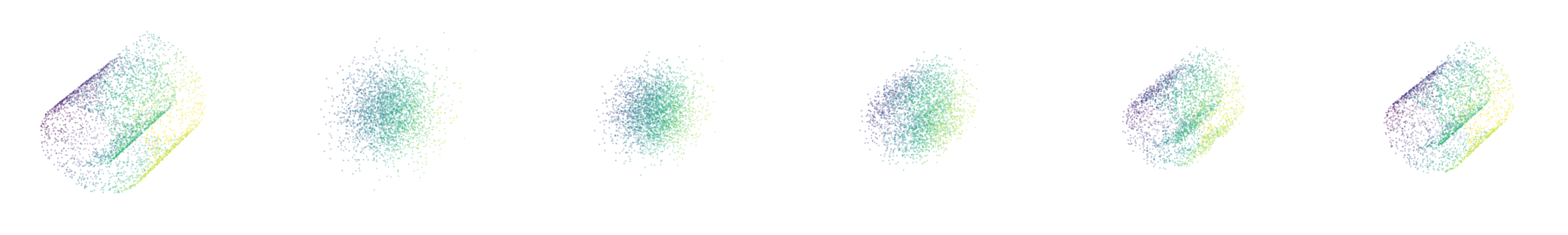}%
};
\begin{scope}[x={(image.south east)}, y={(image.north west)}]
\node[anchor=north, font=\small] at (0.0733,1.10) {Data};
\node[anchor=north, font=\small] at (0.2500,1.10) {$\tau=0.80$};
\node[anchor=north, font=\small] at (0.4167,1.10) {$\tau=0.50$};
\node[anchor=north, font=\small] at (0.5863,1.10) {$\tau=0.25$};
\node[anchor=north, font=\small] at (0.7600,1.10) {$\tau=0.10$};
\node[anchor=north, font=\small] at (0.9267,1.10) {$\tau=0.00$};
\end{scope}
\end{tikzpicture}
\caption{Swiss-roll generation under DMTN in the order FFN$\to$MHSA. The first panel shows the $4000$-point training set and the remaining panels show generated samples at the indicated generative times. All eight conditions generate the data comparably, as the null task requires, and their transports are distinguished in Table~\ref{tab:swiss-scan}.}
\label{fig:swiss-roll-scan}
\end{figure}

All eight models use the same unsupervised flow-matching protocol.\footnote{The data are $4000$ coordinate-standardized Swiss-roll points with Gaussian noise of standard deviation $0.05$. Each context has $N=256$ tokens; the model has one layer, one head, and hidden width $64$. Training uses $50,000$ Adam steps, learning rate $2\times10^{-3}$, batch size $64$, and EMA decay $0.999$, with seed $123$. Generation uses $100$ Heun steps. Sampling without replacement introduces weak token dependence; the null-task argument is exact for independent draws.} Figure~\ref{fig:swiss-roll-scan} shows a representative DMTN trajectory. Table~\ref{tab:swiss-scan} reports the nearest-target distance $d_{\mathrm{NN}}$, effective neighbor count $e^{\bar H}$, detailed-balance residual $\epsilon_{\mathrm{rev}}$, and score skew ratio $\rho_{\mathrm{skew}}$.\footnote{$d_{\mathrm{NN}}$ is the median distance from generated points to their nearest training points; the target's median distinct-neighbor spacing is $0.0494$. $\bar H$ is mean ATTN-row entropy. $\epsilon_{\mathrm{rev}}=\|c-c^\top\|_{\mathrm F}/\|c\|_{\mathrm F}\le\sqrt2$, where $c=\mathrm{diag}(\pi)\mathsf P$ is stationary flux. $\rho_{\mathrm{skew}}$ is a score-level ratio, not a pure cycle invariant. Operators are probed at $\tau=0.05$.} All eight runs give $d_{\mathrm{NN}}$ between $0.0494$ and $0.0512$, close to the native spacing.

\begin{table}[h!]
\centering
\begingroup\scriptsize
\setlength{\tabcolsep}{4pt}
\renewcommand{\arraystretch}{1.06}
\begin{tabular}{llcccc}
\toprule
Operator & Order & $d_{\mathrm{NN}}$ & $e^{\bar H}$ & $\epsilon_{\mathrm{rev}}$ & $\rho_{\mathrm{skew}}$ \\
\midrule
DMAP & FFN$\to$MHSA & $0.0506$ & $27.2$ & $2.46\times10^{-4}$ & $0$ \\
DMAP & MHSA$\to$FFN & $0.0502$ & $1.14$ & $1.83\times10^{-6}$ & $0$ \\
DMTN & FFN$\to$MHSA & $0.0498$ & $42.4$ & $0.224$ & $0.287$ \\
DMTN & MHSA$\to$FFN & $0.0508$ & $11.1$ & $0.234$ & $0.116$ \\
ATAP & FFN$\to$MHSA & $0.0494$ & $46.4$ & $6.15\times10^{-5}$ & $0$ \\
ATAP & MHSA$\to$FFN & $0.0512$ & $9.63$ & $1.27\times10^{-5}$ & $0$ \\
ATTN & FFN$\to$MHSA & $0.0499$ & $56.2$ & $0.439$ & $0.859$ \\
ATTN & MHSA$\to$FFN & $0.0508$ & $10.4$ & $1.414$ & $10.449$ \\
\bottomrule
\end{tabular}
\endgroup
\caption{Swiss-roll scan (seed $123$): similar generated samples, distinct learned transport.}
\label{tab:swiss-scan}
\end{table}

With MHSA$\to$FFN, DMAP approaches an identity-like walk ($e^{\bar H}=1.14$), whereas ATTN develops almost maximally nonreversible flux ($\epsilon_{\mathrm{rev}}\simeq\sqrt2$) and a large skew score. DMTN retains a definite bond but allows circulation, settling near $\epsilon_{\mathrm{rev}}=0.23$ in either order; DMAP and ATAP remain reversible up to numerical error. Thus block order substantially changes the learned transport without comparably changing $d_{\mathrm{NN}}$.

\section{Experiment: Diffusion Transformer}

The Swiss-roll is a neat toy, now we transition to a more realistic setting of \textsc{ImageNet-1k}. Ordinary ATTN imposes no definite-metric constraint, hence we expect in Diffusion-Transformers, trained via \textsc{Adam}, not to yield emergent Riemannian geometry natively.
We begin with \textsc{SiT-XL/2}\footnote{The pretrained \textsc{SiT-XL/2} \citep{ma2024_sit} is a class-conditional diffusion transformer trained on \textsc{ImageNet-1k} at $256\times256$ resolution. Images are encoded using the Stable Diffusion variational autoencoder into $32\times32\times4$ latents. With patch size $2$, the transformer operates on $N=256$ tokens, using bidirectional, unmasked self-ATTN and adaptive layer normalization for class and generative-time conditioning.} to apply our layer-wise ablations. The untouched \textsc{SiT-XL/2} produces results on fig.~\ref{fig:sit} with FID\footnote{FID is computed from $50,000$ class-balanced samples generated using the SiT Euler ODE sampler with $50$ steps and classifier-free guidance scale $1.5$. Generated latents are decoded in \texttt{fp32} using the Stable Diffusion variational autoencoder. Reference statistics are obtained from $50,000$ class-balanced \textsc{ImageNet-1k} latents decoded through the same autoencoder. The Fr\'echet distance is computed between the means and covariances of InceptionV3 features extracted using \textsc{pytorch-fid}, without external resizing \citep{heusel2017_tturgan, seitzer2020_pytorchfid}. The reported uncertainties are $95\%$ bootstrap-normal half-widths from $100$ replicates, resampling both generated and reference features. They quantify evaluation-sample uncertainty, not variation across training seeds. Since the reference images pass through the autoencoder, these FID values are internally comparable but not directly comparable to published SiT scores.} of $2.7650\pm0.1199$.
The first ablation we attempt is removing the indefinite sector yielding DMTN, and obtain the following figures via a vertical-transformation\footnote{Every head of the family lives in one parameter space, and in the coordinates $W_{\pm}=(W_Q\pm W_K)/\sqrt{2}$, with $G_{\pm}=W_{\pm}W_{\pm}^{\top}$ and $B=W_{QK}-W_{QK}^{\top}$, the four operators have bilinear forms $\tfrac12(G_{+}-\epsilon_{-}G_{-}+\epsilon_{B}B)$ with $\epsilon_{-},\epsilon_{B}\in\{0,1\}$. With $\mathcal{Q}=\mathcal{R}W_Q$, $\mathcal{K}=\mathcal{R}W_K$, and $\mathcal{Z}=\mathcal{R}W_{+}=(\mathcal{Q}+\mathcal{K})/\sqrt{2}$, their scores read
\begin{align}\label{eq:dit-lattice}
\mathcal{W}^{\mathrm{ATTN}}&=\mathcal{Q}\mathcal{K}^{\top}, &
\mathcal{W}^{\mathrm{DMTN}}&=\tfrac12\mathcal{Z}\mathcal{Z}^{\top}+\tfrac12\left(\mathcal{Q}\mathcal{K}^{\top}-\mathcal{K}\mathcal{Q}^{\top}\right),\\
\mathcal{W}^{\mathrm{ATAP}}&=\tfrac12\left(\mathcal{Q}\mathcal{K}^{\top}+\mathcal{K}\mathcal{Q}^{\top}\right), &
\mathcal{W}^{\mathrm{DMAP}}&=\tfrac12\mathcal{Z}\mathcal{Z}^{\top}, \nonumber
\end{align}
each followed by the temperature $d_h^{-1/2}$ and a row softmax.}, without training, fig.~\ref{fig:dmtn}(left). However, on application of 
\textsc{AdamW}\footnote{The pretrained SiT-XL/2 checkpoint is transferred to DMTN by removing the negative metric sector $G_-$ while retaining the antisymmetric transport field, without changing the parameter count. The resulting model is adapted for $40,000$ steps on precomputed \textsc{ImageNet-1k} latents using the original SiT linear-path velocity flow-matching objective. Optimization uses \textsc{AdamW} \citep{kingma2015_adam, loshchilov2019_adamw} with constant learning rate $10^{-4}$, batch size $64$, zero weight decay, and gradient-norm clipping at $5$, on a single accelerator in \texttt{TF32} precision. The optimizer state is initialized afresh, and an exponential moving average with decay $0.9999$ is restarted at transfer. All reported samples and FID measurements use the averaged weights.} 
we rapidly obtain clean images, fig.~\ref{fig:dmtn}(right) with FID $2.9341\pm0.1111$.
This is very convincing that a geometric prior is the correct inductive prior. 
DMTN preserves the learned definite geometry, but retains general nonequilibrium transport. What if we also constrain that transport to be reversible?
This brings us to our next ablation the purely exact geometric theory of DMAP. The vertical transition is given in fig.~\ref{fig:dmap}(left) once again yields a mess, however supplementing with \textsc{AdamW}\footnote{Starting from the DMTN checkpoint adapted for $40,000$ steps, we transfer to the distance form of DMAP by retaining the symmetric projection $W_+$ and removing the antisymmetric transport sector. The resulting ATTN scores are $-\frac12\|\mathcal Z_i-\mathcal Z_j\|^2$, with $\mathcal Z=\mathcal R W_+$. Unlike the tied DMAP form, this transformation also changes the effective kernel temperature and introduces a key-norm site potential. The model is then adapted for a further $80,000$ steps using the original SiT flow-matching objective and \textsc{AdamW} with learning rate $10^{-4}$, batch size $64$, zero weight decay, and gradient-norm clipping at $5$. Optimization uses a fresh optimizer state and a restarted exponential moving average of decay $0.9999$. All evaluations use the averaged weights. The DMAP ATTN block has $3D^2$ parameters instead of the original $4D^2$.} 
once again rapidly yields clean images, fig.~\ref{fig:dmap}(right) with FID of $8.1210\pm0.1914$ (notably with only half of the ATTN weights).

\begin{figure}[hbt!]
\centering
\includegraphics[width=0.49\linewidth]{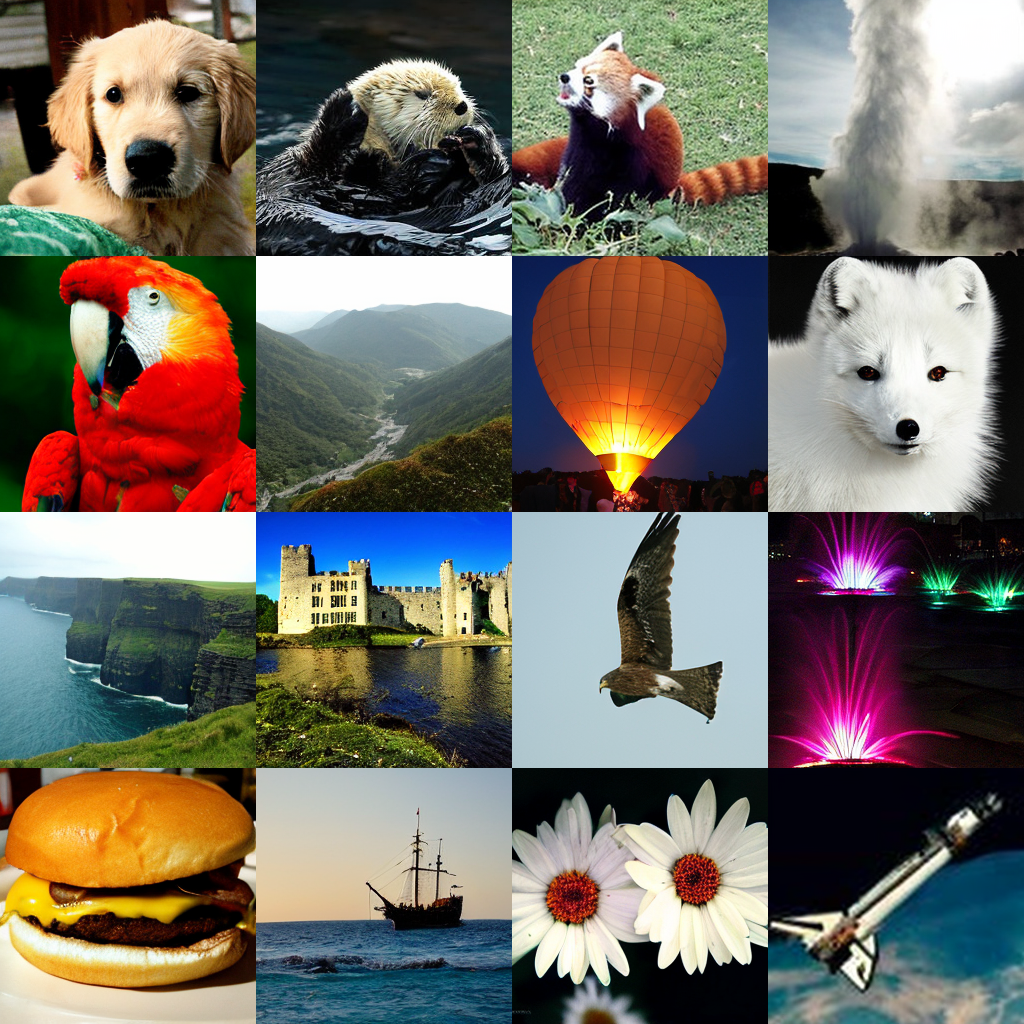}
\caption{Samples from the unmodified pretrained SiT-XL/2 \citep{ma2024_sit}.}
\label{fig:sit}
\end{figure}

\begin{figure}[hbt!]
\centering
\includegraphics[width=0.49\linewidth]{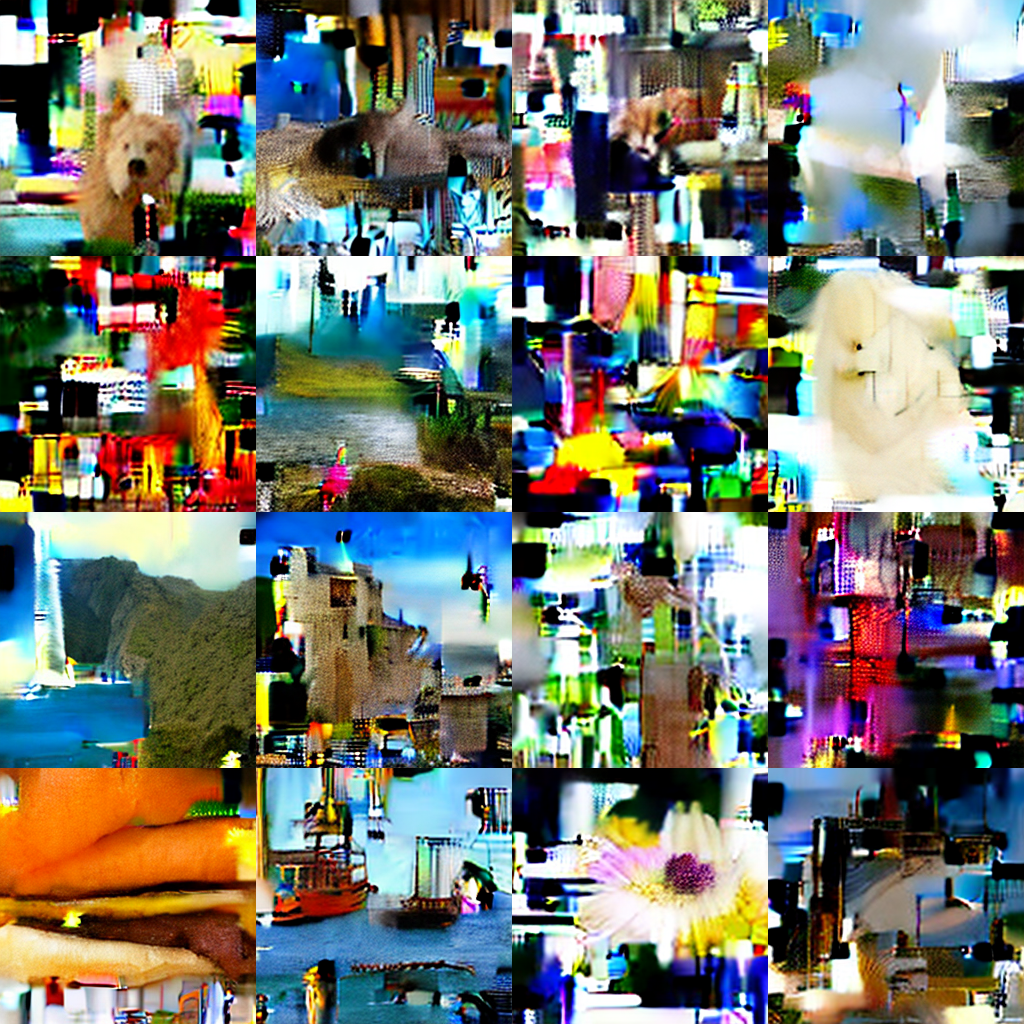}
\includegraphics[width=0.49\linewidth]{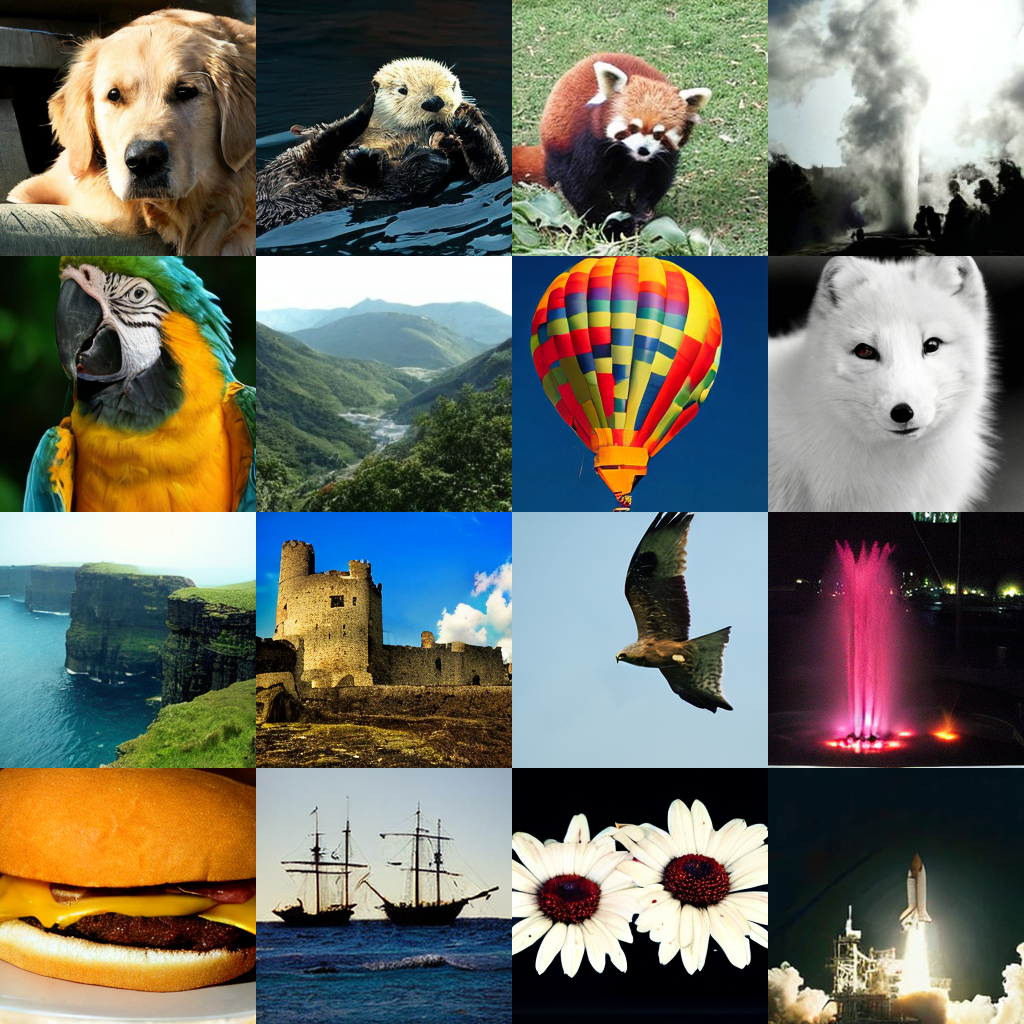}
\caption{DMTN transfer. Samples immediately after substituting $\mathcal{W}^{\mathrm{DMTN}}$ of Eq.~\eqref{eq:dit-lattice} into the pretrained checkpoint (left) and after $40{,}000$ adaptation steps (right).}
\label{fig:dmtn}\label{fig:amap}
\end{figure}

\begin{figure}[hbt!]
\centering
\includegraphics[width=0.49\linewidth]{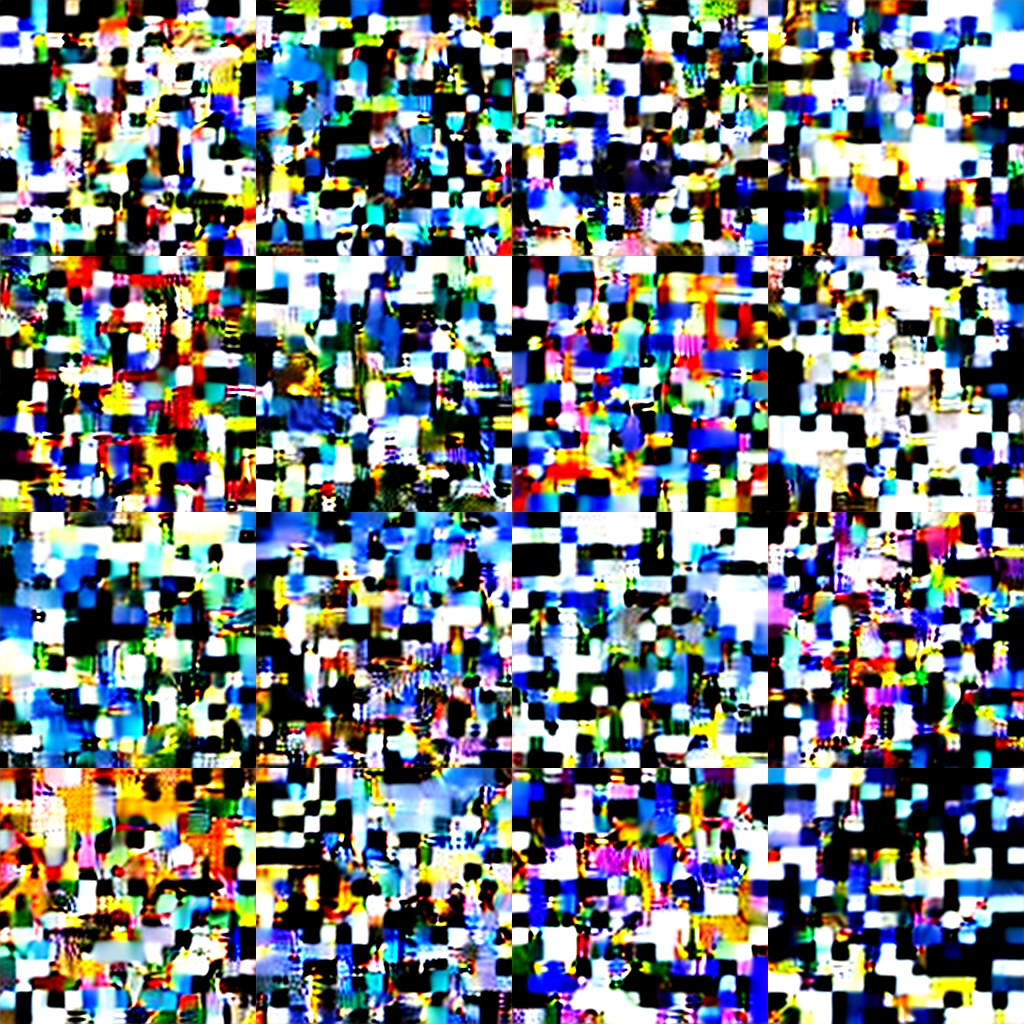}
\includegraphics[width=0.49\linewidth]{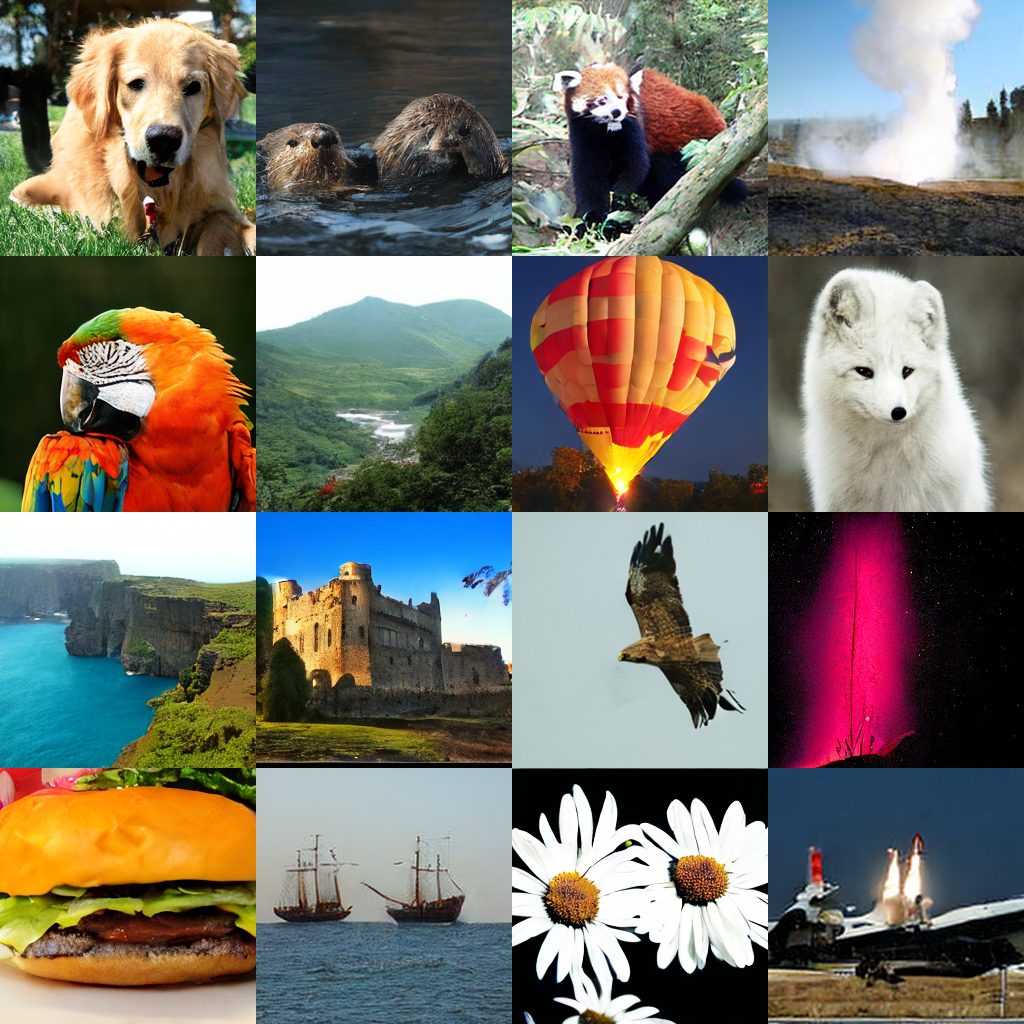}
\caption{Distance-form DMAP transfer. Samples immediately after folding the adapted DMTN checkpoint into the distance form of Remark~\ref{eq:dit-lattice} (left) and after $80{,}000$ further adaptation steps (right).}
\label{fig:dmap}
\end{figure}
\section{Experiment: Autoregressive Transformers}\label{sec:lm}

Unlike the bidirectional SiT, an autoregressive Transformer restricts attention to preceding tokens. Its causal mask makes the resulting Markov operator directed even when the \emph{pre-mask} score is symmetric. We therefore test restrictions on the learned score unrestricted ATTN, DMTN with a definite symmetric bond and a directed link, and DMAP with a definite distance bond but no free directed link.

We train three depth-12 \textsc{nanochat}, \citep{nanochat}, decoders (approximately $200$M parameters) from random initialization with a common tokenizer, data, and optimization budget.\footnote{FineWeb-derived training data; $2{,}500$ steps at $2^{19}$ tokens per step (approximately $1.32\times10^9$ tokens); \texttt{bf16}; Muon for matrix weights and AdamW for embeddings. Each operator was trained with one seed.} Table~\ref{tab:lm-scratch} reports held-out bits-per-byte (bpb). DMTN nearly matches ATTN, differing by only $0.0021$ bpb, whereas DMAP trails ATTN by $0.0150$ bpb. Thus a definite symmetric bond appears compatible with competitive next-token prediction in this setting.

\begin{table}[ht]
\centering
\begin{tabular}{lc}
\toprule
Pre-mask score & Validation bpb $\downarrow$ \\
\midrule
ATTN & $0.8471$ \\
DMTN & $0.8492$ \\
DMAP & $0.8621$ \\
\bottomrule
\end{tabular}
\caption{Depth-12 training from scratch. The measurements are single-seed results; ATAP was not evaluated.}
\label{tab:lm-scratch}
\end{table}

A second experiment probes a more specific capability in a pretrained depth-32 \textsc{nanochat} model (approximately $1.9$B parameters). We sequentially transfer the learned weights through ATTN ($A_0$), DMTN ($M_1$), ATTN ($A_1$), and DMTN ($M_2$), adapting the last three stages for $3{,}000$, $1{,}000$, and $1{,}000$ steps, respectively.\footnote{At each switch the optimizer and learning-rate schedule are restarted. The initial $A_0$ checkpoint predates an architecture change; only $M_1$, $A_1$, and $M_2$ share the revised architecture. CORE scores use $500$ examples per task and are intended for internal comparison.} The results in Table~\ref{tab:lm-switches} suggest that induction is more sensitive to the available score geometry than aggregate language-modeling loss.

\begin{table}[ht]
\centering
\begin{tabular}{llccc}
\toprule
State & Score & bpb $\downarrow$ & Induction $\uparrow$ & CORE $\uparrow$ \\
\midrule
$A_0$ & ATTN & $0.7650$ & $0.8081$ & $0.2853$ \\
$M_1$ & DMTN & $0.6657$ & $0.5735$ & $0.3083$ \\
$A_1$ & ATTN & $\mathbf{0.6619}$ & $\mathbf{0.8334}$ & $\mathbf{0.3275}$ \\
$M_2$ & DMTN & $0.6690$ & $0.4676$ & $0.3082$ \\
\bottomrule
\end{tabular}
\caption{Sequentially adapted depth-32 checkpoints, not independent matched-budget ablations.}
\label{tab:lm-switches}
\end{table}

Returning from DMTN ($M_1$) to ATTN ($A_1$) raises induction accuracy from $0.5735$ to $0.8334$, while bpb improves only from $0.6657$ to $0.6619$. Reimposing DMTN ($M_2$) reduces induction accuracy to $0.4676$ and raises bpb to $0.6690$; CORE changes in the same direction.
\section{Conclusion}\label{sec:conclusion}

The central lesson is that attention admits an \emph{operator geometry}, and that this geometry can be read. At finite $N$, every query--key score splits exactly into a symmetric bond and an antisymmetric link, with endpoint factors supplying site potentials after rewriting the symmetric score as a distance. The split applies to any trained head, without probes, auxiliary training, or approximation, and turns an opaque attention matrix into a learned metric together with directed transport upon it. Attention thereby returns Markov operators to language modeling \citep{markov1906_dependentQuantities, shannon1948_mathematicalTheoryCommunication}, now acting on the context itself.

DMAP, DMTN, ATAP, and unrestricted ATTN expose two independent choices: whether the feature bond is definite and whether the link permits directed circulation (Table~\ref{tab:family}). Each distinction is measurable on a trained head. Bond definiteness tells whether a head measures a metric or contrasts competing associations, cycle affinities and stationary currents locate directed circulation, the \textit{carr\'e-du-champ} $\Gamma$ reveals which pairs a head treats as near, remote partners included, and $\Gamma_2$ and $\Omega_\pi$ quantify its contraction and circulation.

\begin{table}[t]\centering
\begin{tabular}{l|cc}
 & $\mathcal{A}$ exact & $\mathcal{A}$ free\\ \hline
$\mathcal{S}$ definite & DMAP & DMTN\\
$\mathcal{S}$ indefinite & ATAP & ATTN
\end{tabular}
\caption{The transport family. Each operator is $\operatorname{rownorm}\exp_\odot(\mathcal{S}+\mathcal{A})$, sorted by whether its symmetric logit keeps the definite quadratic core and whether its antisymmetric logit stays exact. DMAP keeps both rules with half of attention's query--key parameters, DMTN and ATAP each break one, and ATTN breaks both.}
\label{tab:family}
\end{table}

In the continuum, the local sector leads to Einstein's diffusion, its learned Riemannian metric, and the Witten--Girsanov description of drift. Where off-diagonal transition rates survive the diffusive scaling, Beurling--Deny and Courr\`ege provide the language for reversible and directed jumps. The finite score and flux identities are exact; the continuum jump limits remain conditional on localization and rate-scaling assumptions, and a general head need not converge to a local differential operator, or even to a finite-rate jump generator. Non-locality is nonetheless native to attention, and the Bakry--\'Emery calculus makes the resulting viewpoint operational: for local diffusions, $\Gamma$ and $\Gamma_2$ recover familiar metric and curvature information.

Our experiments offer preliminary support for treating these restrictions as meaningful inductive biases, and they show what reading a head can reveal. On the Swiss roll, all four operators generate samples at the native spacing while learning markedly different transport, a difference visible only through the operator. On \textsc{ImageNet} latent-flow generation, a DMTN adaptation approaches the unrestricted SiT baseline, whereas the more constrained DMAP adaptation incurs a larger fidelity penalty. In autoregressive language modeling, DMTN nearly matches unrestricted attention in validation bits per byte, while sequential ATTN--DMTN switches coincide with much larger changes in induction accuracy, suggesting that induction relies both on directed circulation and non-local interactions.

Prior theoretical work on transformers and directed graphs has operated at the affinity level, decomposing asymmetric kernels additively after the nonlinearity \citep{perraultjoncas2011_directedGraphEmbedding,ting2010_convergenceGraphLaplacians,wright2021_transformersNonMercerKernels}. Moving the decomposition before the nonlinearity changes what is visible, because the antisymmetric sector then acts as a multiplicative field.

This work began at weak coupling. With $S=\mathcal Q\mathcal K^\top=\mathcal R M\mathcal R^\top$, row-wise softmax admits the expansion
\begin{equation}\label{eq:conclusion-high-temperature}
\mathrm{softmax}_\beta(S)
=N^{-1}\mathbf1\mathbf1^\top
+\frac{\beta}{N}S\bigl(\mathsf I-N^{-1}\mathbf1\mathbf1^\top\bigr)+O(\beta^2),\qquad \beta\to0.
\end{equation}
At first order the nonlinearity drops out, and attention, linearized diffusion maps, and PCA reduce to a common row-centered bilinear operator governed by random-matrix spectra \citep{elkaroui2010_kernelRandomMatrices,letey2025_pretrainTestTaskAlignment,candanedo2025_linearizedDiffusionMap}. The same limit shows how directed transport enters. Centering the context, $R_i=\mu+\delta_i$, splits the antisymmetric score exactly,
\begin{equation}\label{eq:conclusion-centered-split}
\tfrac12\bigl(S_{ij}-S_{ji}\bigr)
=\underbrace{\phi_j-\phi_i}_{\text{exact}}
+\underbrace{\tfrac12\,\delta_iA\,\delta_j^\top}_{\text{circulating}},
\end{equation}
$\phi_j=\tfrac12\,\mu A\,\delta_j^\top,$ with $A=M-M^\top$. The mean contributes only a Witten site potential, one conservative field seen alike by every token, which row normalization absorbs into scalar reweighting. Then circulation is carried by the fluctuations alone, which is exactly the $A$-weighted oriented area of the triangle,
\begin{equation}\label{eq:conclusion-cycle-area}
\log\frac{\mathsf P_{ij}\mathsf P_{jk}\mathsf P_{ki}}{\mathsf P_{ji}\mathsf P_{kj}\mathsf P_{ik}}
=\beta\,(\delta_j-\delta_i)\,A\,(\delta_k-\delta_i)^\top .
\end{equation}
An antisymmetric form vanishes on a line, so fluctuations along a single principal direction collapse DMTN exactly onto DMAP. Whether indefinite ATAP bonds likewise flow to DMAP, completing an infrared universality class, is left to future work.

\newpage
\bibliographystyle{apalike}
\bibliography{refs}

\newpage
\appendix

\end{document}